\documentclass{article}

\PassOptionsToPackage{numbers, compress}{natbib}

\usepackage[preprint]{neurips_2020}

\usepackage[utf8]{inputenc} 
\usepackage[T1]{fontenc}    
\usepackage{hyperref}       
\usepackage{url}            
\usepackage{booktabs}       
\usepackage{amsfonts}       
\usepackage{nicefrac}       
\usepackage{microtype}      

\usepackage{times}
\usepackage{soul}
\usepackage{graphicx}
\usepackage{amsmath}
\usepackage{multirow}
\usepackage{enumitem} 
\usepackage{amsfonts}
\usepackage{algorithm, algorithmic}
\usepackage{amsthm, amssymb}
\usepackage{relsize}
\usepackage{color}
\usepackage{xcolor}
\usepackage{subcaption}
\usepackage{enumitem}
\usepackage[numbers]{natbib}
\newtheorem{theorem}{Theorem}
\newtheorem{lemma}{Lemma}
\newtheorem{proposition}{Proposition}
\usepackage{threeparttable}
\allowdisplaybreaks[1]

\title{Towards Generalized Implementation of \\ Wasserstein Distance in GANs}

\author {
    Minkai Xu\textsuperscript{\rm 1,2},
    Zhiming Zhou\textsuperscript{\rm 3},
    Guansong Lu\textsuperscript{\rm 4}, 
    Jian Tang\textsuperscript{\rm 2,5,6}, 
    Weinan Zhang\textsuperscript{\rm 4}, Yong Yu\textsuperscript{\rm 4} \\
    \textsuperscript{\rm 1}University of Montreal
    \textsuperscript{\rm 2}Mila - Quebec AI Institute\\
    \textsuperscript{\rm 3}Shanghai University of Finance and Economics
    \textsuperscript{\rm 4}Shanghai Jiao Tong University\\
    \textsuperscript{\rm 5}HEC Montreal 
    \textsuperscript{\rm 6}CIFAR AI Research Chair \\
    minkai.xu@umontreal.ca,  heyohaizhou@gmail.com\\
    \{gslu,wnzhang,yyu\}@apex.sjtu.edu.cn,
    jian.tang@hec.ca
}

\begin{document}

\maketitle

\begin{abstract}
    Wasserstein GANs (WGANs), built upon the Kantorovich-Rubinstein (KR) duality of Wasserstein distance, is one of the most theoretically sound GAN models. However, in practice it does not always outperform other variants of GANs. This is mostly due to the imperfect implementation of the Lipschitz condition required by the KR duality. Extensive work has been done in the community with different implementations of the Lipschitz constraint, which, however, is still hard to satisfy the restriction perfectly in practice. In this paper, we argue that the strong Lipschitz constraint might be unnecessary for optimization. Instead, we take a step back and try to relax the Lipschitz constraint. Theoretically, we first demonstrate a more general dual form of the Wasserstein distance called the \emph{Sobolev duality}, which relaxes the Lipschitz constraint but still maintains the favorable gradient property of the Wasserstein distance. Moreover, we show that the KR duality is actually a special case of the Sobolev duality. Based on the relaxed duality, we further propose a generalized WGAN training scheme named Sobolev Wasserstein GAN (SWGAN), and empirically demonstrate the improvement of SWGAN over existing methods with extensive experiments.\footnote{Code is available at \url{https://github.com/MinkaiXu/SobolevWassersteinGAN}.}
\end{abstract}
\section{Introduction}

Generative adversarial networks (GANs) \cite{gan} have attracted huge interest in both academia and industry communities due to its effectiveness in a variety of applications.
Despite its effectiveness in various tasks, a common challenge for GANs is the training instability \cite{gan-tutorial}.
In literature, many works have been developed to mitigate this problem \cite{principled_methods,lucic2017gans,heusel2017gans,mescheder2017numerics,mescheder2018training,yadav2017stabilizing}. 

By far, it is well known that the problem of training instability of the original GANs mainly comes from the ill-behaving
distance metric \cite{principled_methods}, \textit{i.e.}, the Jensen-Shannon divergence metric, which remains constant when two distributions are disjoint. The Wasserstein GAN \cite{wgan} improves this by using the Wasserstein distance, which is able to continuously measure the distance between two distributions. Such a new objective has been shown to be effective in improving training stability. 

In practice, since the primal form of the Wasserstein distance is difficult to optimize, the WGAN model \cite{wgan} instead proposed to optimize it with the Kantorovich-Rubinstein (KR) duality \cite{krdual}. However, though the new WGAN scheme is theoretically more principled, it does not yield better performance in practice compared to other variants of GANs \cite{lucic2017gans}. The main obstacle is that the WGAN requires the discriminator (or the critic) to be a Lipschitz function. However, this is very hard to satisfy in practice though a variety of different implementations have been tried such as weight clipping \cite{wgan}, gradient penalty (GP) \cite{wgan-gp}, Lipschitz penalty (LP) \cite{wgan-lp} and spectral normalization (SN) \cite{sngan}. As a result, WGAN is still unable to always achieve very compelling results.

In this paper, we argue that the strong Lipschitz condition might be unnecessary in the inner optimization loop for WGAN's critic. Intuitively, a looser constraint on the critic, which results in a larger function space, can simplify
the practical constrained optimization problem of the restricted critic, and the better-trained critic would further benefit the training of the generator. Therefore, instead of developing new methods to impose the Lipschitz constraint, in this paper we propose to relax this constraint. In other words, we move our attention from ``how to better implement the Lipschitz constraint'' to ``how to loosen the Lipschitz constraint''. More specifically, in this paper we demonstrate a new dual form of the Wasserstein distance where the Lipschitz constraint is relaxed to the Sobolev constraint \cite{2003sobolev,sgan}. We further show that the new duality with the relaxed constraint indeed is a generalization of the KR duality, and it still keeps the gradient property of the Wasserstein distance. Based on this relaxed duality, we propose a generalized WGAN model called Sobolev Wasserstein GAN. To the best of our knowledge, among the restricted GAN models \cite{wgan,wgan-gp,sgan,fishergan,bellemare2017cramer,adler2018banach}, 
Sobolev Wasserstein GAN is the most relaxed one that can still avoid the training instability problem. 



The main contributions of this paper can be summarized as follows: 


\begin{itemize}
\vspace{-5pt}
\item We demonstrate the Sobolev duality of Wasserstein distance and demonstrate that the new duality is also capable of alleviating the training instability problem in GANs (Section~\ref{section-sdual}). We further clarify the relation between Sobolev duality and other previous metrics and highlight that by far Sobolev duality is the most relaxed metric that can still avoid the non-convergence problem in GANs (Section~\ref{wgan-sec}).
\vspace{-3pt}
\item Based on Sobolev duality, we introduce the Inequality Constraint Augmented Lagrangian Method \cite{alm} to build the practical Sobolev Wasserstein GAN (SWGAN) training algorithm (Section \ref{section-swgan}). 
\vspace{-5pt}
\end{itemize}
We conduct extensive experiments to study the practical performance of SWGAN. We find that generally our proposed model achieves better sample quality and is less sensitive to the hyper-parameters. We also present a theoretical analysis of the minor sub-optimal equilibrium problem common in WGAN family models, and further propose an improved SWGAN with a better convergence.

\section{Preliminaries}
\label{preliminaries}

\subsection{Generative adversarial networks}

Generative adversarial networks \cite{gan} perform generative modeling via two competing networks. The generator network $G$ learns to map samples from a noise distribution to a target distribution, while the discriminator network $D$ is trained to distinguish between the real data and the generated samples. Then the generator $G$ is trained to output images that can fool the discriminator $D$. The process is iterated. 
Formally, the game between the generator $G$ and the discriminator $D$ leads to the minimax objective: 
\begin{equation}
\label{JSD}
\begin{aligned}
    \min_G \max_D \Big \{ &\mathbb{E}_{x \sim P_r}[\log(D(x))] +
    &\mathbb{E}_{z \sim P_z}[\log(1-D(G(z))] \Big \},
\end{aligned}
\end{equation}
where $P_r$ denotes the distribution of real data and $P_z$ denotes the noise distribution. 

This objective function is proven to be equivalent to the Jensen-Shannon divergence (JSD) between the real data distribution $P_r$ and fake data distribution $P_g$ when the discriminator is optimal. Assuming the discriminator is perfectly trained, the optimal discriminator is as follows: 
\begin{equation}
\label{gan-discriminator}
    D^*(x) = \frac{P_r}{P_r+P_g}.
\end{equation} 
However,  recently \cite{LGAN} points out that the gradients provided by the optimal discriminator in vanilla GAN cannot consistently provide meaningful information for the generator's update, which leads to the notorious training problems in GANs such as gradient vanishing \cite{gan,principled_methods} and mode collapse \cite{mode_gan,unrolled_gan,kodali2017convergence,arora2017generalization}. This view would be clear when checking the gradient of optimal discriminator in Eq.~\eqref{gan-discriminator}: the value of the optimal discriminative function $D^*(x)$ at each point is independent of other points and only reflects the local densities of $P_r(x)$ and $P_g(x)$, thus, when the supports of the two distributions are disjoint, the gradient produced by a well-trained discriminator is uninformative to guide the generator \cite{LGAN}. 

\subsection{Wasserstein distance}
Let $P_r$ and $P_g$ be two data distributions on $\mathbb{R}^n$. The Wasserstein-$1$ distance between $P_r$ and $P_g$ is defined as
\begin{align}
\label{emd}
    W(P_r,P_g) = \inf_{\pi \in \Pi(P_r,P_g)} \mathbb{E}_{(x_i,x_j)\sim\pi}[\Vert x_i-x_j \Vert],
\end{align}
where the coupling $\pi$ of $P_r$ and $P_g$ is the probability distribution on $\mathbb{R}^n \times \mathbb{R}^n$ with marginals $P_r$ and $P_g$, and $\Pi(P_r,P_g)$ denotes the set of all joint distributions $\pi$. The Wasserstein distance can be interpreted as the minimum cost of transporting one probability distribution to another. 
The Kantorovich-Rubinstein (KR) duality \cite{krdual} provides a new way to evaluate the Wasserstein distance between distributions. The duality states that 
\begin{equation}
\label{krdual}
\begin{aligned}
W&(P_r,P_g) = {\sup}_{f} \, \Big \{ \mathbb{E}_{x \sim P_r} [f(x)] - \mathbb{E}_{x \sim P_g}  [f(x)] \Big \} , \\
&\emph{s.t.  } \, f(x_i) - f(x_j) \leq \Vert x_i-x_j \Vert, \,\, \forall x_i, \forall x_j.
\end{aligned}
\end{equation}
where the supremum is taken over all 
functions $f$: $\mathcal{X} \rightarrow \mathbb{R}$ whose Lipschitz constant is no more than one.

\subsection{Wasserstein GAN}
\label{ot-section}
The training instability issues of vanilla GAN is considered to be caused by the unfavorable property of distance metric \cite{principled_methods}, \textit{i.e.}, the JSD remains constant when the two distributions are disjoint. Accordingly, \cite{wgan} proposed Wasserstein distance in the form of KR duality Eq.~\eqref{krdual} as an alternative objective.

The Wasserstein distance requires to enforce the Lipschitz condition on the critic network $D$. It has been observed in previous work that imposing Lipschitz constraint in the critic leads to improved stability and sample quality \cite{wgan,kodali2017convergence,fedus2017many,farnia2018convex}. Besides, some researchers also found that applying Lipschitz continuity condition to the
generator can benefit the quality of generated samples \cite{zhang2018self,odena2018generator}.
Formally, in WGAN family, with the objective being Wasserstein distance, the optimal critic $f^*$ under Lipschitz constraint holds the following property \cite{wgan-gp}:
\begin{proposition}
\label{grad-direc}
    Let $\pi^*$ be the optimal coupling in Eq.~\eqref{emd}, then the optimal function $f^*$ in KR duality  Eq.~\eqref{krdual} satisfies that: let $x_t = tx_i + (1-t)x_j$ with $0 \leq t \leq 1$, if $f^*$ is differentiable and $\pi^*(x, x)=0$ for all $x$, then it holds that $P_{(x_i, x_j)\sim \pi^*}[\nabla f^* (x_t)=\frac{x_i-x_j}{\Vert x_i-x_j \Vert}]=1$. 
\end{proposition}
This property indicates that for each coupling of generated datapoint $x_j$ and real datapoint $x_i$ in $\pi^*$, the gradient at any linear interpolation between $x_i$ and $x_j$ is pointing towards the real datapoint $x_i$ with unit norm.
Therefore, guided by the gradients, the generated sample $x_j$ would move toward the real sample $x_i$. This property provides the explanation, from the gradient perspective, on why WGAN can overcome the training instability issue. 


\subsection{Sobolev space}
Let $\mathcal{X}$ be a compact space in $\mathbb{R}^n$ and let $\mu$(x) to be a distribution defined on $\mathcal{X}$ as a dominant measure. Functions in the Sobolev space $ W^{1,2}(\mathcal{X},\mu)$ \cite{2003sobolev} can be written as: 
\begin{equation}
\begin{aligned}
	W^{1,2}(\mathcal{X},\mu) = \Big\{ f:\mathcal{X} \rightarrow \mathbb{R}, \int_{\small{\mathcal{X}}} \Vert \nabla_x f(x)\Vert ^2 \mu (x) dx < \infty \Big\}.
\end{aligned}
\end{equation}
Restrict functions to the Sobolev space $W^{1,2}(\mathcal{X},\mu)$ vanishing at the boundary and denote this space by $W^{1,2}_0(\mathcal{X},\mu)$, then the semi-norm in  $W^{1,2}_0(\mathcal{X},\mu)$ can be defined as:
\begin{align}
   \Vert f \Vert _{W^{1,2}_0(\mathcal{X},\mu)} = \sqrt {\int _\mathcal{X} \Vert \nabla _x f(x) \Vert ^2 \mu (x) dx}.
\end{align}
Given the notion of semi-norm, we can define the Sobolev unit ball constraint as follows: 
\begin{equation}
\label{sobolev-unit-ball}
\begin{aligned}
    \mathcal{F}_S(\mathcal{X},\mu) = \Big\{& f:\mathcal{X} \rightarrow \mathbb{R}, \,f \in W^{1,2}_0(\mathcal{X},\mu),  \\
    &\Vert f \Vert _{W^{1,2}_0(\mathcal{X},\mu)} \leq 1 \Big\}.
\end{aligned}
\end{equation}
Sobolev unit ball is a function class that restricts the square root of integral of squared gradient norms according to the dominant measure $\mu(x)$.

\subsection{Sobolev GAN}
\label{sobolev-gan}
After WGAN, many works are devoted to improving GAN model by imposing restrictions on the critic function. Typical instances are the GANs based on Integral Probability Metric (IPM) \cite{fishergan,bellemare2017cramer}. Among them, Sobolev GAN (SGAN) \cite{sgan} proposed using Sobolev IPM as the metric for training GANs, which restricts the critic network $D$ in $\mathcal{F}_S(\mathcal{X},\mu)$:
\begin{equation}
\begin{aligned}
\label{sipm}
    \mathcal{S}_{\mu}(P_r,P_g) = \sup_{f \in \mathcal{F}_S(\mathcal{X},\mu)} \Big \{\mathbb{E}_{x \sim P_r} [f(x)] - \mathbb{E}_{x \sim P_g}  [f(x)] \Big \}.
\end{aligned}
\end{equation}
The following choices of measure $\mu$ for $\mathcal{F}_S$ are considered, which we will take as our baselines:
\begin{enumerate}[label=(\alph*)]
	\vspace{-3pt}
    \item $\mu =\frac{P_r+P_g}{2}$: the mixed distribution of $P_r $ and $P_g$; 
    \vspace{-2pt}
    \item $\mu_{gp}$: $x= t x_i +(1-t)x_j$, where $x_i \sim P_r$, $x_j \sim P_g$ and $t \sim \text{U}[0,1]$, \textit{i.e.}, the distribution defined by the interpolation lines between $P_r$ and $P_g$ as in \cite{wgan-gp}.
    \vspace{-3pt}
\end{enumerate}
Let $F_{P_r}$ and $F_{P_g}$ be the cumulative distribution functions (CDF) of $P_r$ and $P_g$ respectively, and assume that the $n$ partial derivatives of $F_{P_r}$ and $F_{P_g}$ exist and are continuous. Define the differential operator $D^- = (D^{-1},...,D^{-n})$ where $D^{-i}$ = $\dfrac{\partial^{n-1}}{\partial x_1...\partial x_{i-1}\partial x_{i+1}...\partial x_n}$, which computes $(n-1)$ high-order partial derivative excluding the $i$-th dimension. Let $x^{-i} = (x_1,...,x_{i-1},x_{i+1},...,x_d)$. According to \cite{sgan}, the Sobolev IPM in Eq.~\eqref{sipm} has the following equivalent form:
\begin{equation}
\small
\label{sobolev-ipm-dual}
\mathcal{S}_{\mu}(P_r, P_g) =  \frac{1}{n} \sqrt{\mathbb{E}_{x\sim \mu} \sum_{i=1}^n \bigg ( \frac{D^{-i} F_{P_r}(x) -  D^{-i} F_{P_g}(x)}{\mu(x)} \bigg )^2}.
\end{equation}
Note that, for each $i$, $D^{-i} F_{P}(x)$ is the cumulative distribution of the variable $X_i$ given the other variables $X^{-i} = x^{-i}$ weighted by the density function of $X^{-i}$ at $x^{-i}$, \emph{i.e.}, 
\begin{equation}
\label{conditional}
\begin{aligned}
D^{-i} F_{P}(x) = P_{[X^{-i}]}(x^{-i}) F_{P_{[X_i|X^{-i}=x^{-i}]}}(x_i).
\end{aligned}
\end{equation}
Thus, the Sobolev IPM can be seen as a comparison of coordinate-wise conditional CDFs. 
Furthermore, \cite{sgan} also proves that the optimal critic $f^*$ in SGAN holds the following property:
\begin{equation}
\small
\label{sobolev-f}
\nabla_x f^*(x)= \frac{1}{n \mathcal{S}_{\mu}(P_r,P_g)} \frac{D^- F_{P_g}(x)- D^-F_{P_r}(x)}{\mu(x)}.
\end{equation}

\section{Sobolev duality of Wasserstein distance}

\subsection{Sobolev duality}
\label{section-sdual}

Let $x_i$ and $x_j$ be two points in $\mathbb{R}^n$. The linear interpolation between $x_i$ and $x_j$ can be written as $x = t x_i + (1-t)x_j$ with $0 \le t \le 1$. Regarding $x$ as a random variable on the line between $x_i$ and $x_j$, we can then define its probability distribution as $\mu^{x_i,x_j}(x)$, which we will later use as the dominant measure for Sobolev space. Formally, let $t$ be the random variable that follows the uniform distribution U$[0, 1]$. Then $\mu^{x_i,x_j}(x)$ can be written as:
\begin{equation}
\label{sample-mu}
    \mu^{x_i,x_j}(x) = 
    \begin{cases}
    \dfrac{1}{\Vert x_i - x_j \Vert}, & x = t x_i + (t-1)x_j,
    \vspace{4pt}
    \cr
    0, &otherwise.
    \end{cases}
\end{equation}
With the above defined notation, we propose our new dual form of Wasserstein distance as follows, which we call \emph{Sobolev duality}\footnote{We provide the proofs of \emph{Sobolev duality} and Proposition \ref{grad-sdual} in Appendix.}
\begin{equation}
\label{s-dual}
\begin{aligned}
    W&(P_r,P_g) = {\sup}_{f} \, \Big \{\mathbb{E}_{x \sim P_r} [f(x)] - \mathbb{E}_{x \sim P_g}  [f(x)] \Big \}, \\
    &\emph{s.t.  } \, f \in \mathcal{F}_S (\mathcal{X}, \mu ^{x_i,x_j}), \,\, \forall x_i \sim P_r, \forall x_j \sim P_g,
\end{aligned}
\end{equation}
where 
$\mathcal{F}_S (\mathcal{X}, \mu ^{x_i,x_j})$ denotes 
the Sobolev unit ball of 
\begin{equation}
\label{sobolev_unit_ball}
    \Vert f \Vert _{W^{1,2}_0 (\mathcal{X}, \mu ^{x_i,x_j})} = \sqrt{ \int _\mathcal{X} \Vert \nabla _x f(x) \Vert ^2 \mu ^{x_i,x_j} (x) dx } \leq 1.
\end{equation}
Note that the support of $\mu ^{x_i,x_j} (x)$ is the straight line between $x_i$ and $x_j$. Thus $\Vert f \Vert _{W^{1,2} (\mathcal{X}, \mu ^{x_i,x_j})}$ is the square root of the path integral of squared gradient norms from $x_i$ to $x_j$. In other words, the constraint is restricting the gradient integral on each line between $P_r$ and $P_g$ to be no more than 1. 

Corresponding to Proposition~\ref{grad-direc} of KR duality, we highlight the following property of the proposed Sobolev duality: 
\begin{proposition}
\label{grad-sdual}
    Let $\pi^*$ be the optimal coupling in Eq.~\eqref{emd}, then the optimal function $f^*$ in Sobolev duality Eq.~\eqref{s-dual} satisfies that: let $x_t = tx_i + (1-t)x_j$ with $0 \leq t \leq 1$, if $f^*$ is differentiable and $\pi^*(x, x)=0$ for all $x$, then it holds that $P_{(x_i, x_j)\sim \pi^*}[\nabla f^* (x_t)=\frac{x_i-x_j}{\Vert x_i-x_j \Vert}]=1$. 
\end{proposition}



That is, with Sobolev duality, the gradient direction for every fake datum is the same as WGAN. Hence, enforcing the Sobolev duality constraint on discriminator can be an effective alternative of the Lipschitz condition to guarantee a stable training for GAN model. 



\subsection{Relation to other metrics}
\label{wgan-sec}

\textbf{Relation to KR duality in Eq.~(\ref{krdual}).}
As indicated by Proposition~\ref{grad-sdual}, the optimal critic $f^*$ of Sobolev duality actually holds the same gradient property as KR duality in Proposition~\ref{grad-direc}.
However, as clarified below, the constraint in Sobolev duality is indeed looser than KR duality, which would potentially benefit the optimization. 

In the classic KR duality
, $f$ is restricted under Lipschitz condition, \textit{i.e.}, the gradient norms of \textit{all points} in the metric space are enforced to no more than $1$. By contrast, in our Sobolev duality, we restrict \textit{the integral} of squared gradient norms \textit{over each line} between $P_r$ and $P_g$. This implies that Lipschitz continuity is a sufficient condition of the constraint in Sobolev duality. In summary, Sobolev duality is a generalization of KR duality where the constraint is relaxed, while still keeps the same property of training stability. 

\textbf{Relation to Sobolev IPM in Eq.~(\ref{sipm}).}
We now clarify the difference between Sobolev IPM in Eq.~\eqref{sipm} and Sobolev duality of Wasserstein distance in Eq.~\eqref{s-dual}. In the former metric, when implementing $\mu_{gp}$ (defined in Section \ref{sobolev-gan}), the \textit{total integral} of squared gradient norms on all interpolation lines between $P_r$ and $P_g$ is enforced to no more than $1$; while in the latter metric, 
the integral \textit{over each interpolation line} between $P_r$ and $P_g$ is restricted. Therefore, Sobolev duality enforces stronger constraint than Sobolev IPM. 

However, we should also note that the stronger constraint is necessary to ensure the favorable gradient property in Proposition~\ref{grad-sdual}. By contrast, as shown in Eq.~\eqref{sobolev-f}, Sobolev IPM measures coordinate-wise conditional CDF, which cannot always provide gradients as good as the optimal transport plan in Wasserstein distance. A toy example is provided in Appendix
to show the case that Sobolev IPM is sometimes insufficiently constrained to ensure the convergence.

\section{Sobolev Wasserstein GAN}
\label{section-swgan}

Now we define the GAN model with Sobolev duality, which we name as Sobolev Wasserstein GAN (SWGAN). Formally, SWGAN can be written as:
\begin{equation}
\begin{aligned}
\label{loss}
    \min_G \max_D \mathcal{L}_S&(D_w,G_\theta)
    & = \mathbb{E}_{x \sim P_r}D_w(x) - \mathbb{E}_{z \sim P_z}D_w(G_\theta(z)),
\end{aligned}
\end{equation}
with the constraint that
\begin{align}
    &\mathbb{E}_{x\sim \mu ^{x_i,x_j}}\Vert \nabla _x D_w(x)\Vert ^2 \leq 1, \forall x_i \sim P_r, \forall x_j \sim P_g, 
\end{align}
where $\mu ^{x_i,x_j}$ is the interpolation distribution on lines between pairs of points $x_i$ and $x_j$ as defined in Eq.~\eqref{sample-mu}. 



Let $\Omega_{ij}$ denote $1 - \mathbb{E}_{x\sim \mu ^{x_i,x_j}}\Vert \nabla _x D_w(x)\Vert ^2$, then the constraint is to restrict $\Omega_{ij}$ to be greater than or equal to 0 for all the pairs of $(x_i, x_j)$. Inspired by \cite{sgan}, we define the following Augmented Lagrangian inequality regularization \cite{alm} corresponding to SWGAN ball constraints:
\begin{equation}
\label{loss-alm}
\begin{aligned}
    &\mathcal{L}^{(ij)}_{al}(w,\theta,\alpha) = \alpha (\Omega_{ij} - s_{ij}) - \frac{\rho}{2} (\Omega_{ij} - s_{ij})^2, \\
    &\mathcal{L}_{al}(w,\theta,\alpha) = \mathbb{E}_{x_i \sim P_r} \mathbb{E}_{x_j \sim P_g} \mathcal{L}^{(ij)}_{al}(w,\theta,\alpha).
\end{aligned}
\end{equation}
where $\alpha$ is the Lagrange multiplier, $\rho$ is the quadratic penalty weight and $s_{ij}$ represents the slack variables. Practically, $s_{ij}$ is directly substituted by its optimal solution: 
\begin{equation}
    s^*_{ij} = \max\Big\{\Omega_{ij}-\frac{\alpha}{\rho}\,,\,0\Big\}.
\end{equation}
As in \cite{wgan} and \cite{sgan}, the regularization term in Eq.~\eqref{loss-alm} is added to the loss only when training the critic. To be more specific, the training process is: given the generator parameters $\theta$, we train the discriminator by maximizing $\mathcal{L}_S + \mathcal{L}_{al}$; then given the discriminator parameters $w$, we train the generator via minimizing $\mathcal{L}_S$. We leave the detailed training procedure in Appendix.
\section{Experiments}
\label{experiment}

We tested SWGAN on both synthetic density modeling and real-world image generation task.

\begin{figure*}[!t]
	\centering
	\begin{subfigure}{0.11\linewidth}
		\includegraphics[width=0.99\columnwidth]{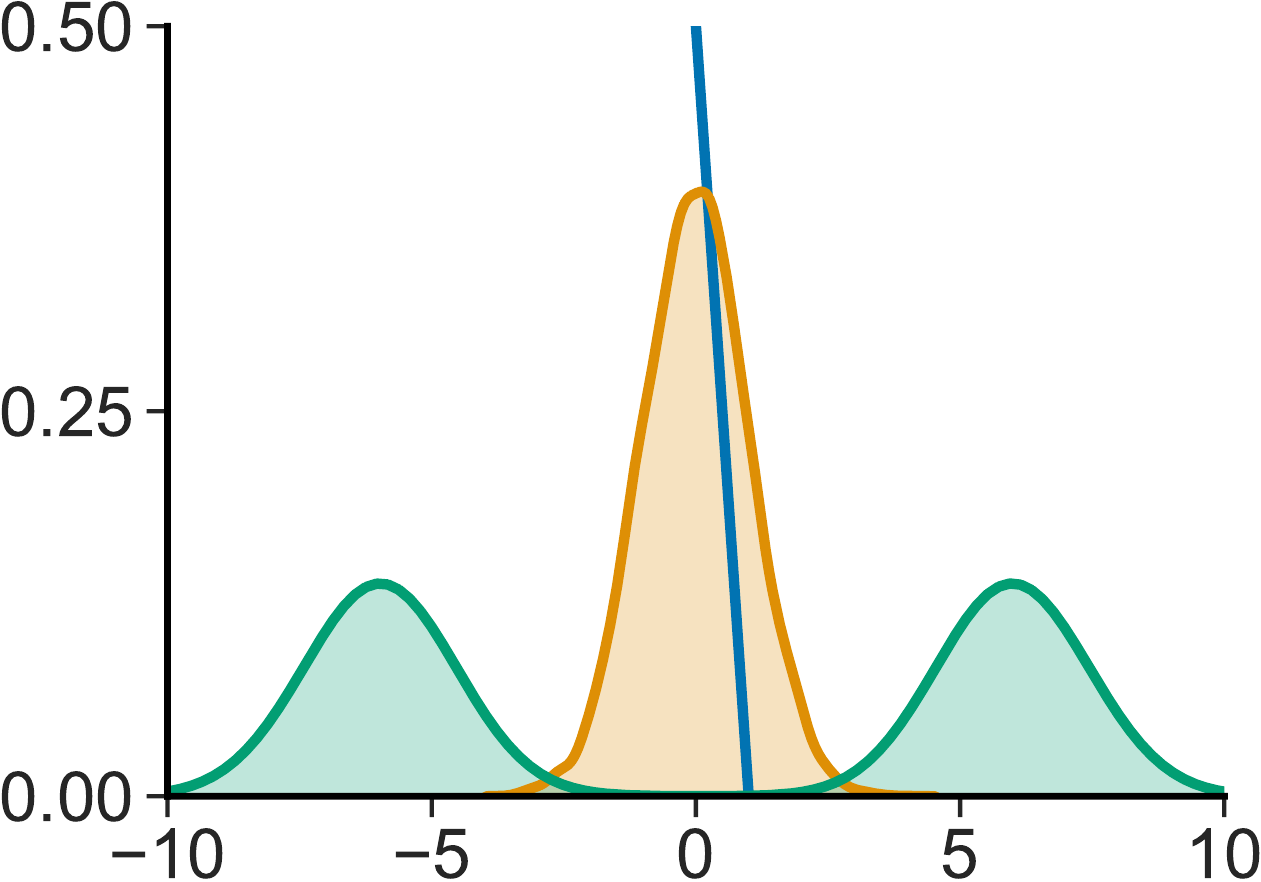}
		\vspace{-5pt}
		\label{1}
	\end{subfigure}	
	\begin{subfigure}{0.11\linewidth}
		\includegraphics[width=0.99\columnwidth]{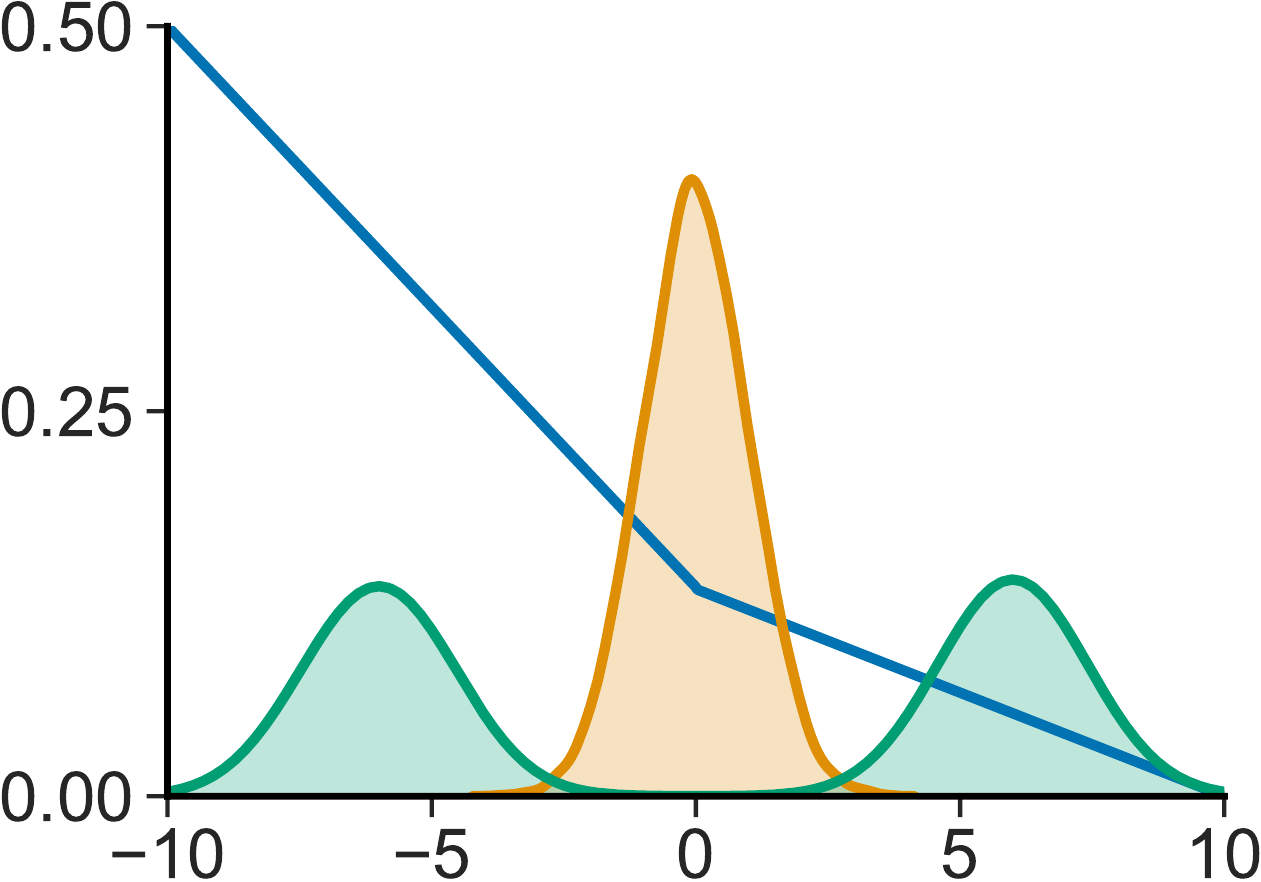}
		\vspace{-5pt}
		\label{1-1}
	\end{subfigure}	
	\begin{subfigure}{0.11\linewidth}
		\includegraphics[width=0.99\columnwidth]{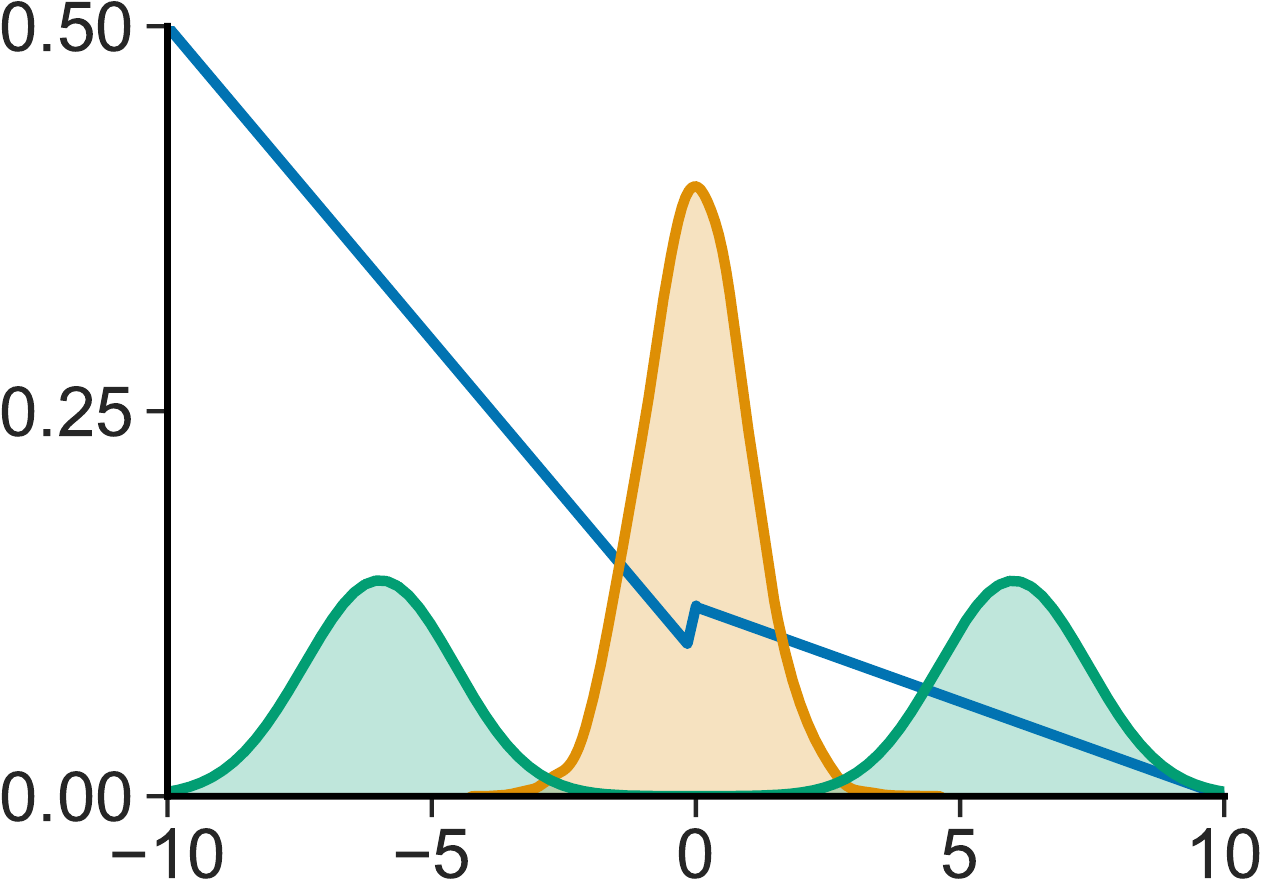}
		\vspace{-5pt}
		\label{2}
	\end{subfigure}
	\begin{subfigure}{0.11\linewidth}
		\includegraphics[width=0.99\columnwidth]{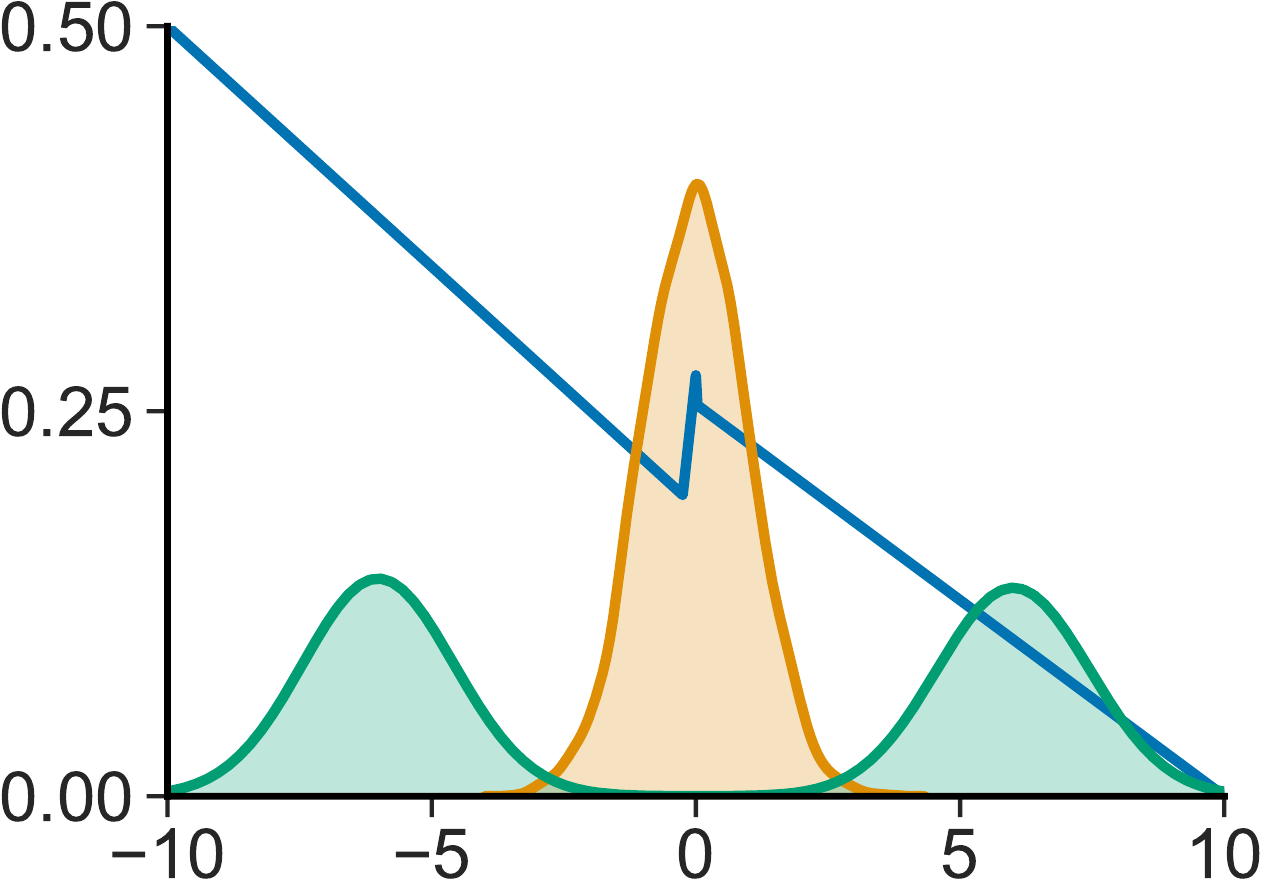}
		\vspace{-5pt}
		\label{2-1}
	\end{subfigure}	
	\begin{subfigure}{0.11\linewidth}
		\includegraphics[width=0.99\columnwidth]{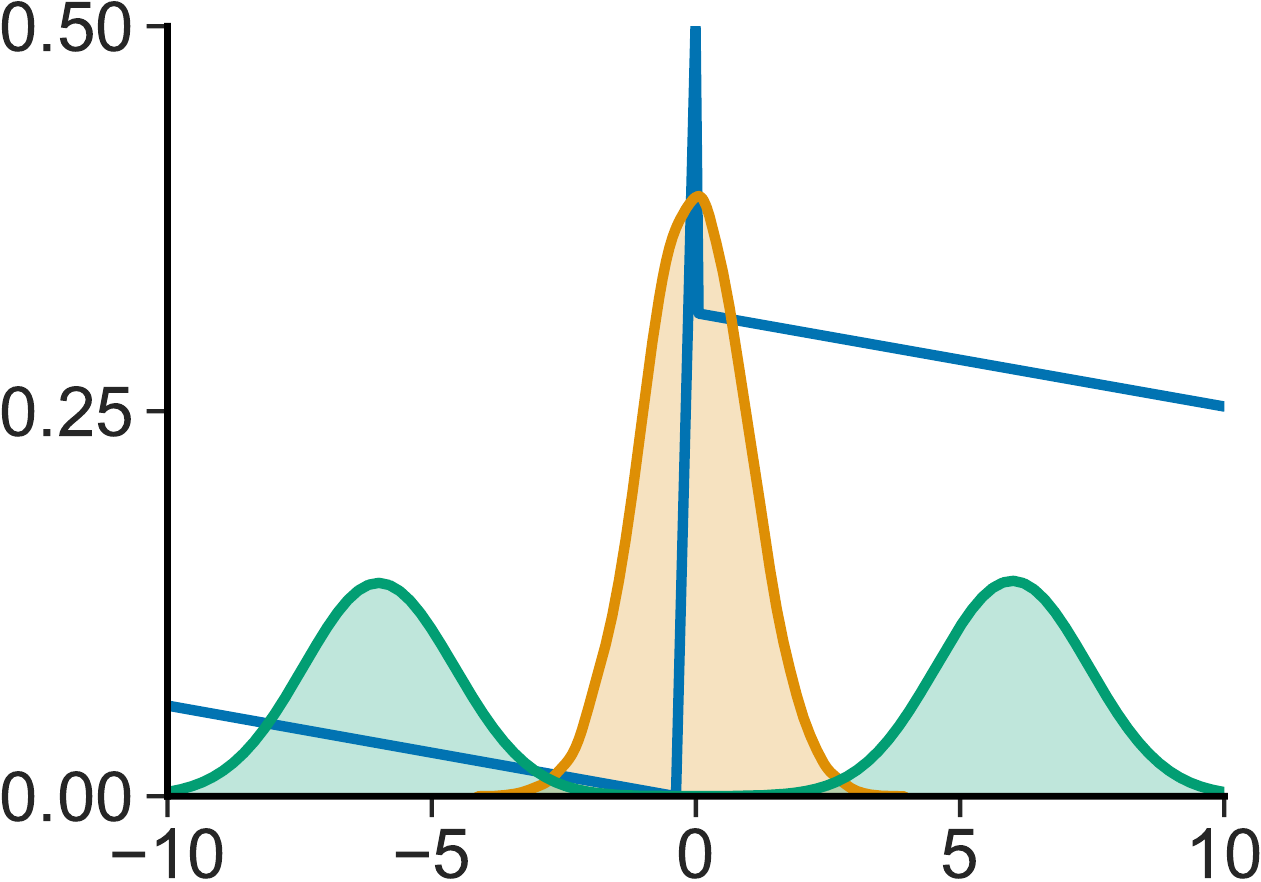}
		\vspace{-5pt}
		\label{3}
	\end{subfigure}
	\begin{subfigure}{0.11\linewidth}
		\includegraphics[width=0.99\columnwidth]{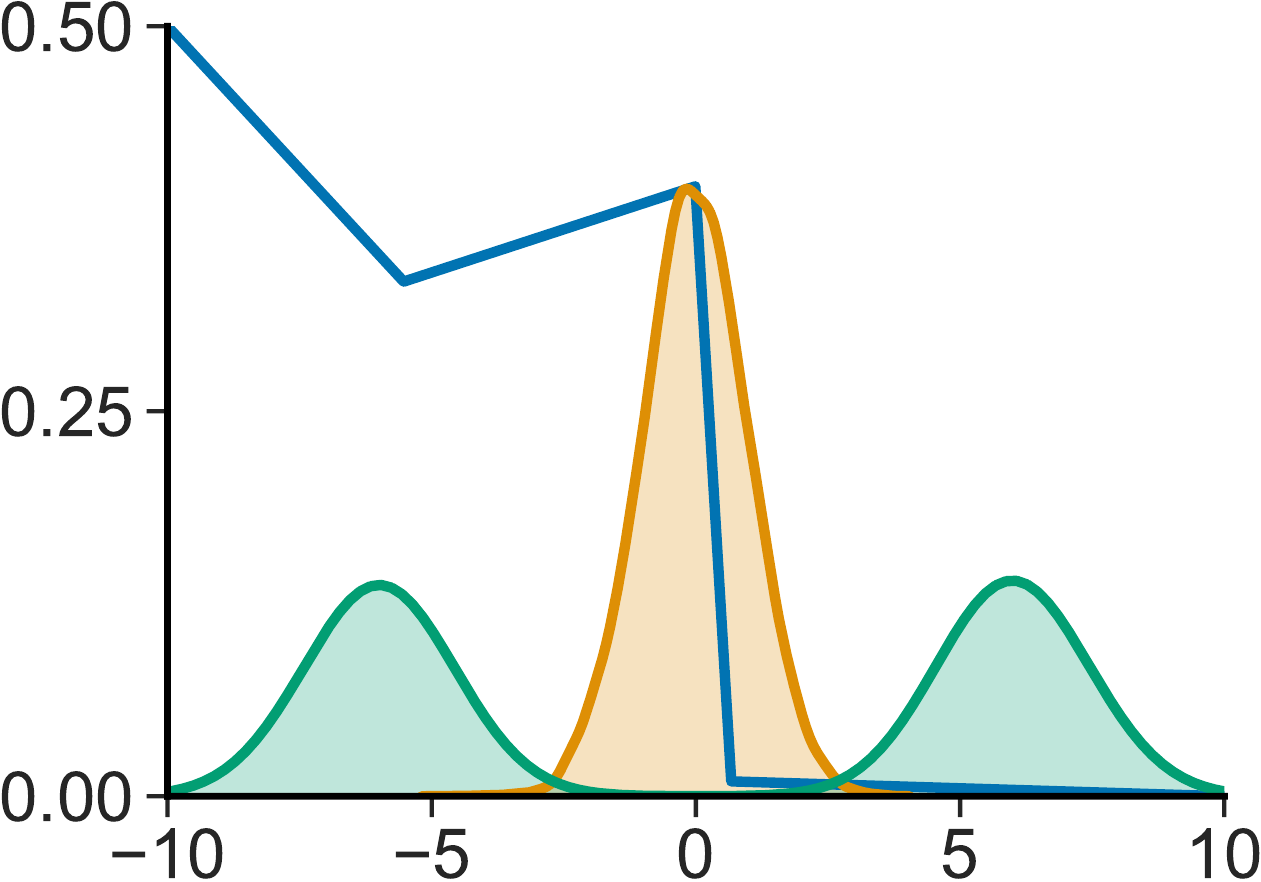}
		\vspace{-5pt}
		\label{4}
	\end{subfigure}
	\begin{subfigure}{0.11\linewidth}
		\includegraphics[width=0.99\columnwidth]{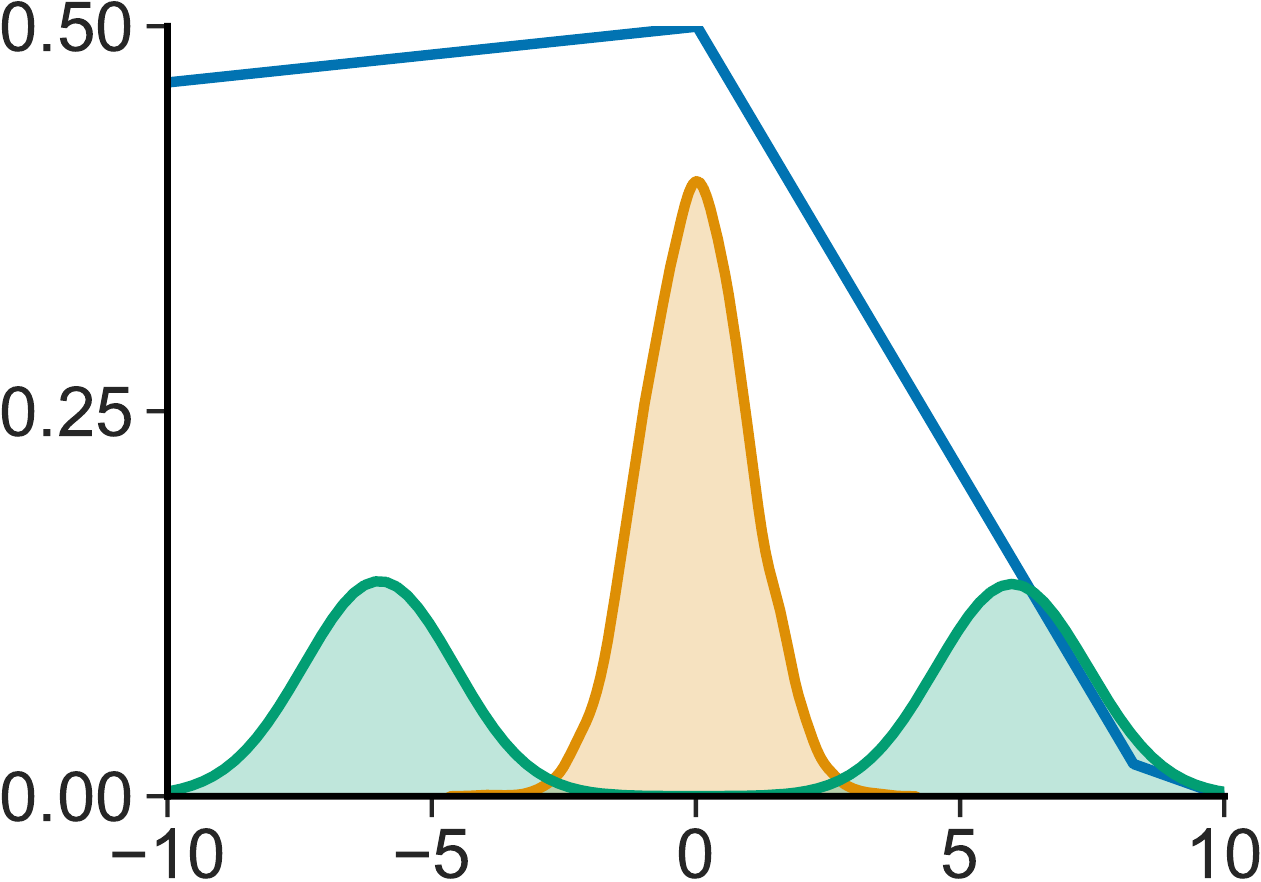}
		\vspace{-5pt}
		\label{5}
	\end{subfigure}
	\begin{subfigure}{0.11\linewidth}
		\includegraphics[width=0.99\columnwidth]{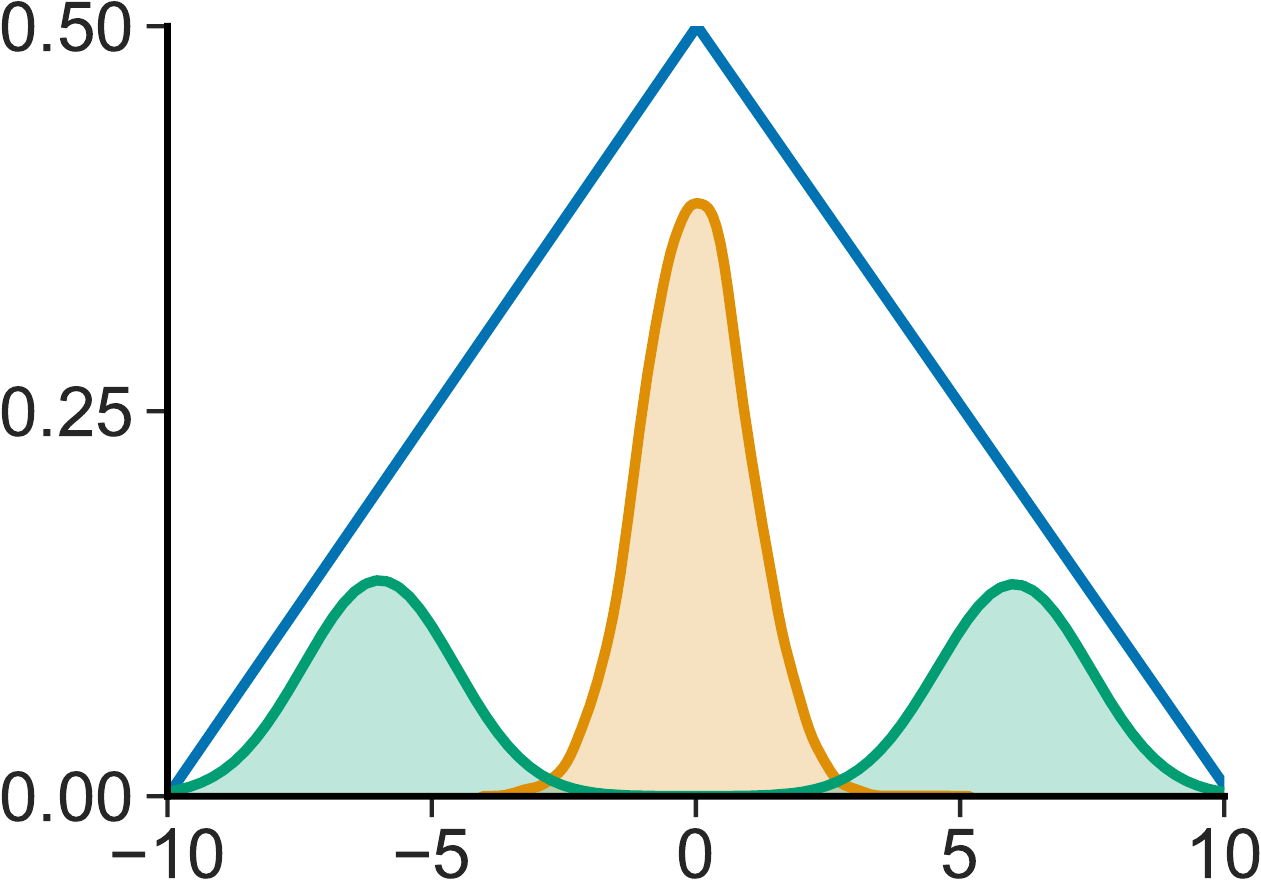}
		\vspace{-5pt}
		\label{6}
	\end{subfigure}	
	
	\begin{subfigure}{0.11\linewidth}
		\includegraphics[width=0.99\columnwidth]{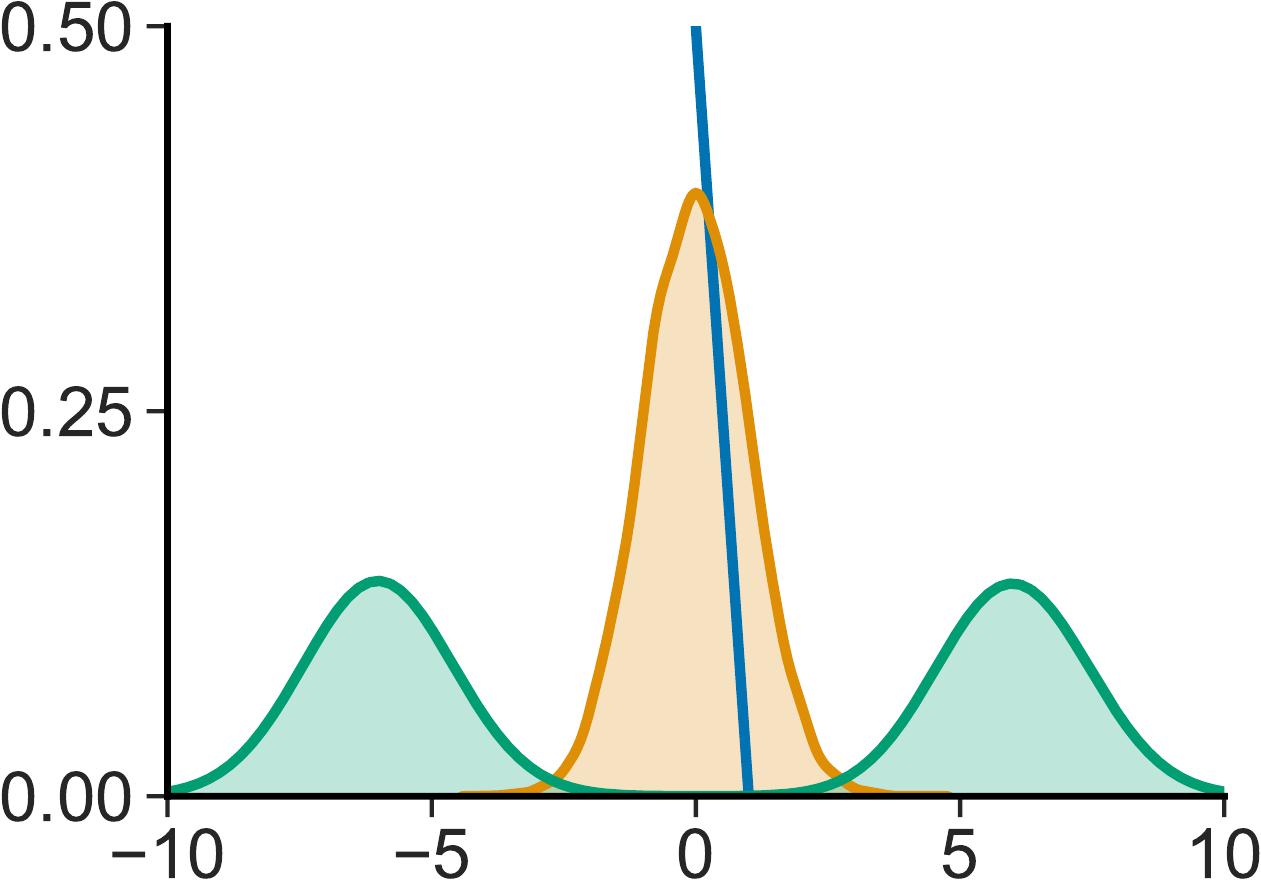}
		\vspace{-10pt}
		\label{7}
	\end{subfigure}
	\begin{subfigure}{0.11\linewidth}
		\includegraphics[width=0.99\columnwidth]{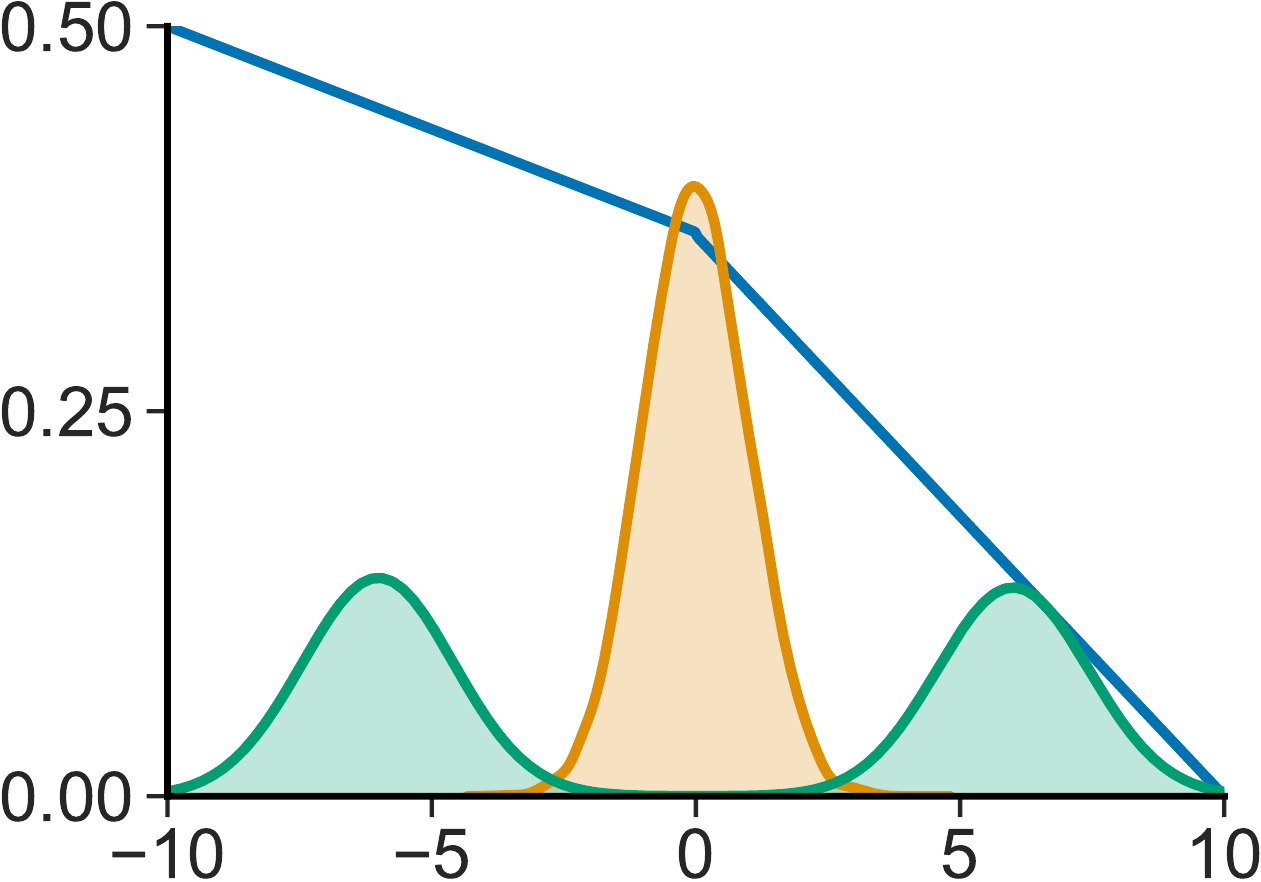}
		\vspace{-10pt}
		\label{7-1}
	\end{subfigure}	
	\begin{subfigure}{0.11\linewidth}
		\includegraphics[width=0.99\columnwidth]{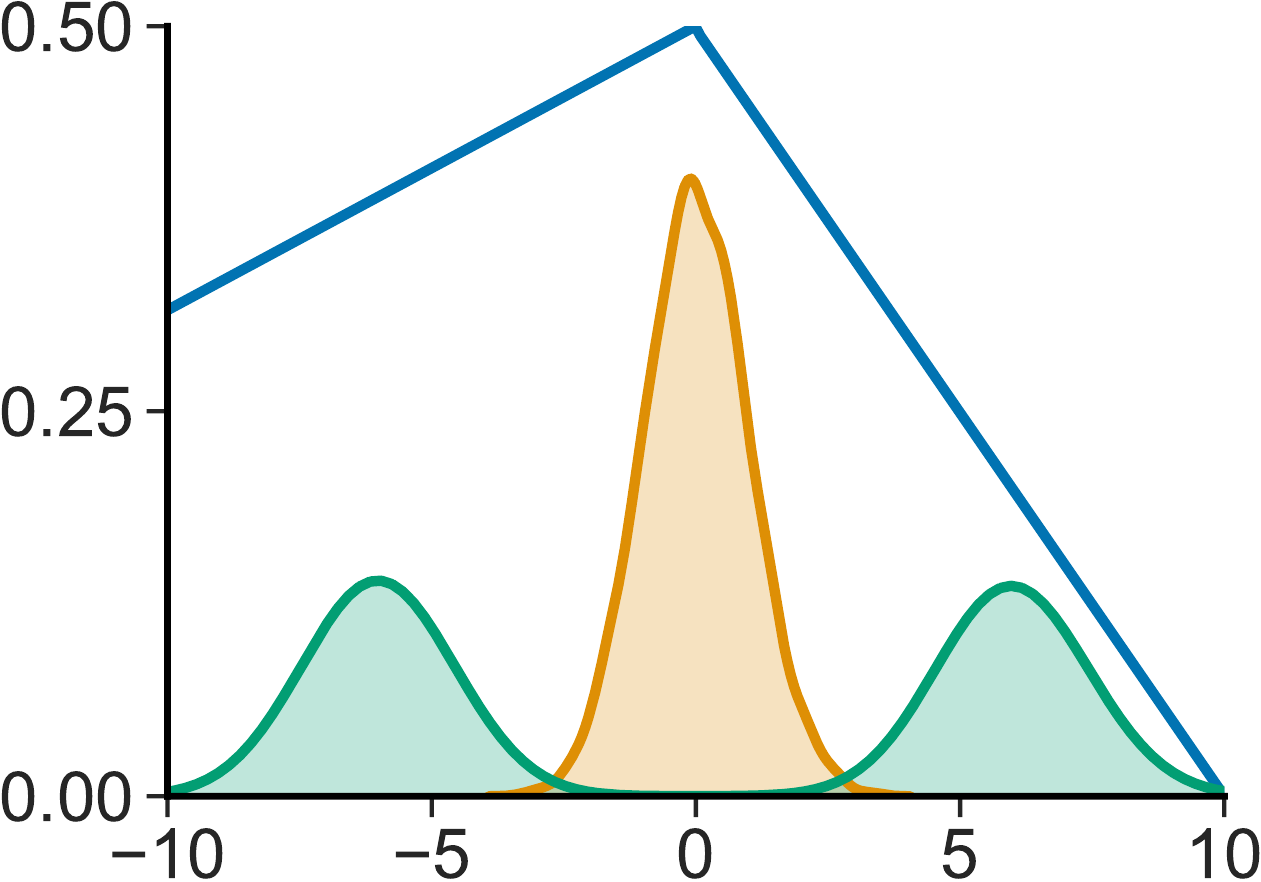}
		\vspace{-10pt}
		\label{8}
	\end{subfigure}
	\begin{subfigure}{0.11\linewidth}
		\includegraphics[width=0.99\columnwidth]{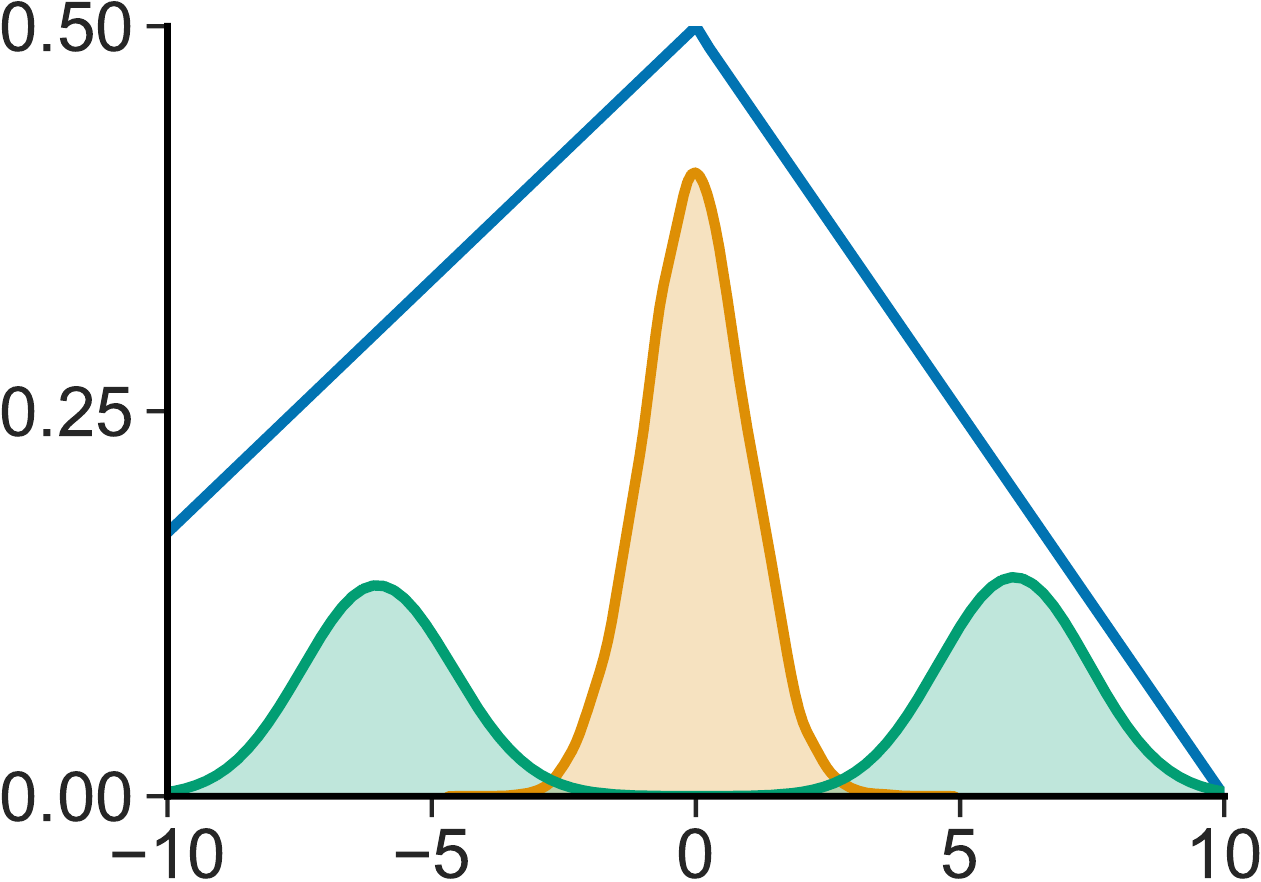}
		\vspace{-10pt}
		\label{8-1}
	\end{subfigure}	
	\begin{subfigure}{0.11\linewidth}
		\includegraphics[width=0.99\columnwidth]{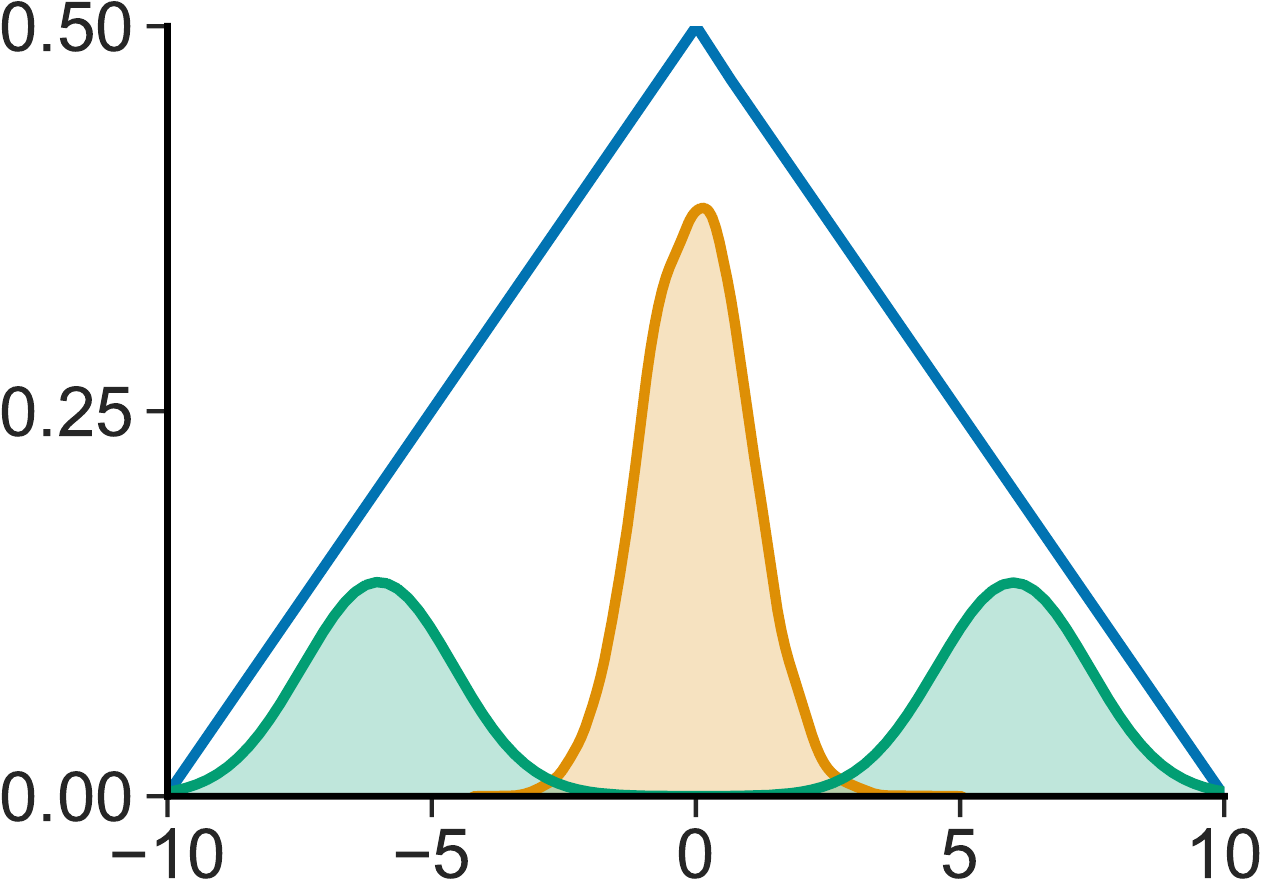}
		\vspace{-10pt}
		\label{9}
	\end{subfigure}
	\begin{subfigure}{0.11\linewidth}
		\includegraphics[width=0.99\columnwidth]{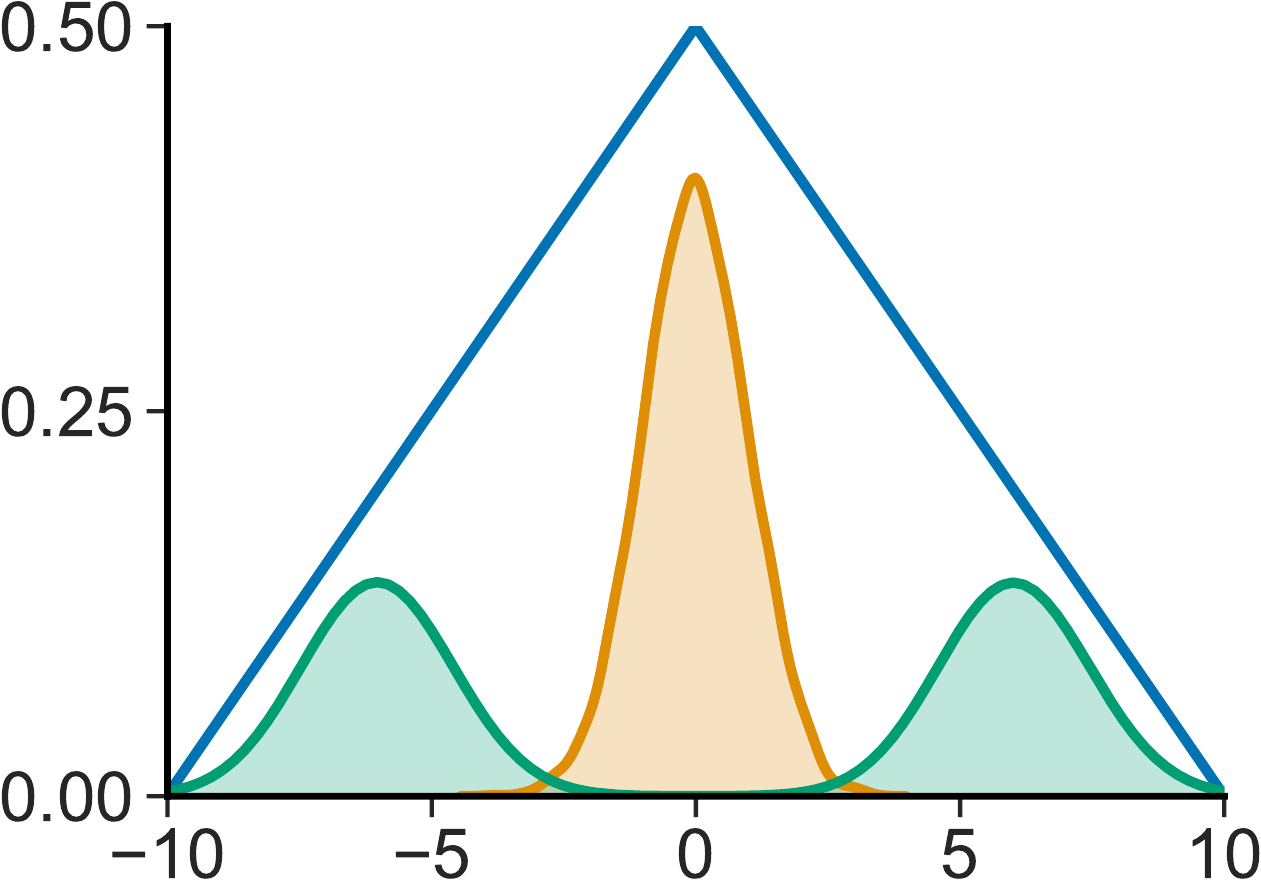}
		\vspace{-10pt}
		\label{10}
	\end{subfigure}
	\begin{subfigure}{0.11\linewidth}
		\includegraphics[width=0.99\columnwidth]{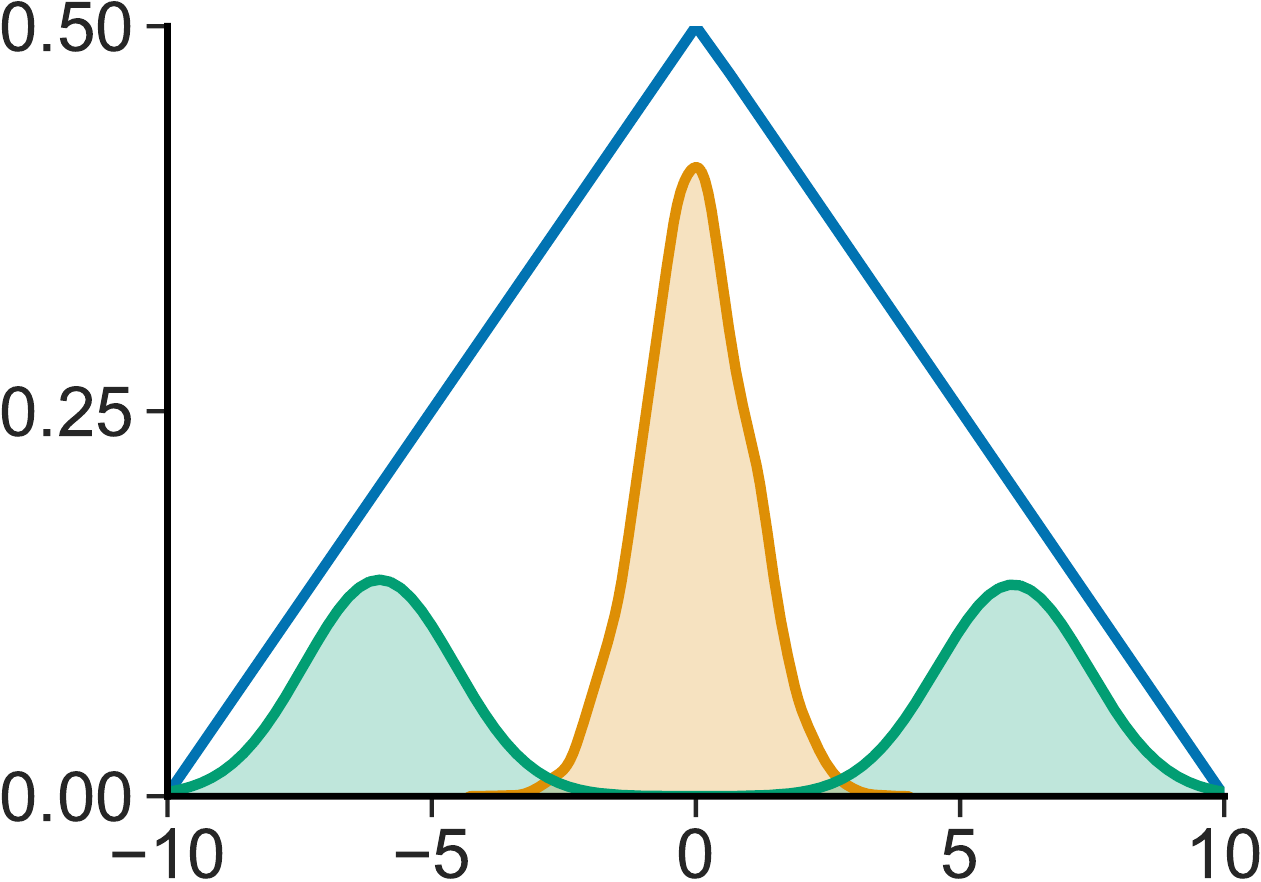}
		\vspace{-10pt}
		\label{11}
	\end{subfigure}
	\begin{subfigure}{0.11\linewidth}
		\includegraphics[width=0.99\columnwidth]{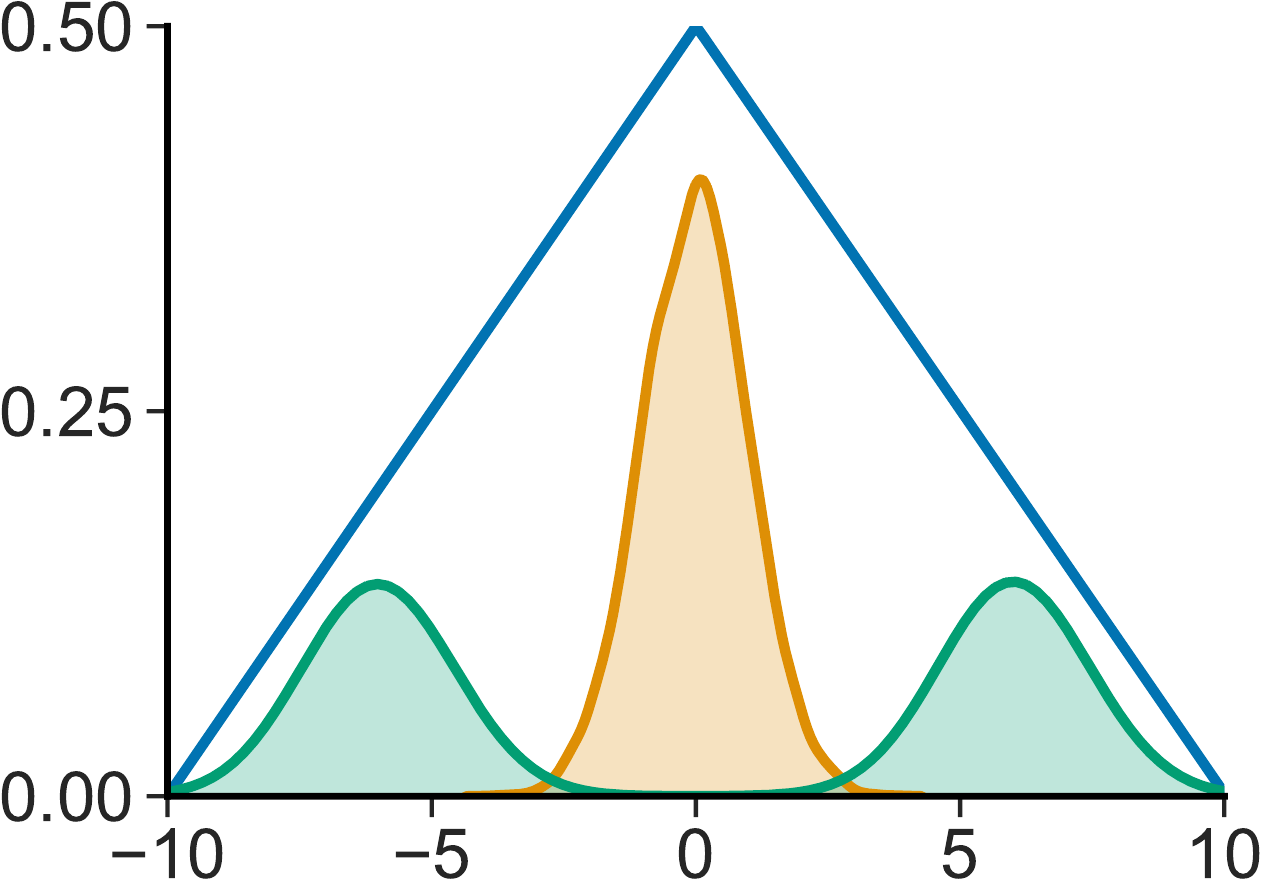}
		\vspace{-10pt}
		\label{12}
	\end{subfigure}
	\caption{1D Training comparison. Upper: WGAN. Lower: SWGAN. Orange: real data, sampled from $N(0,1)$. Green: fake data, sampled from $\frac{1}{2}(N(-5,1)+N(5,1))$. Blue: the re-scaled critic, which is normalized to $[0, 0.5]$. From left to right, different columns correspond to iteration $0, 30, 60, 90, 120, 180, 240, 300$ respectively. 
	The critic of SWGAN holds a faster and smoother convergence. 
	}
	\label{fig:1d_gan}
\end{figure*}

\subsection{Synthetic density modeling} 


\textbf{1D distribution modeling.}
Displaying the level sets is a standard qualitative approach to evaluate the learned critic function for two-dimensional data sets \cite{wgan-gp,kodali2017train,wgan-lp}. Here, we consider both real data distribution $P_r$ and generated fake distribution $P_g$ are fixed simple one-dimensional Gaussian distributions. Our goal is to investigate whether the critics of both WGAN and SWGAN can be efficiently optimized to provide the favorable gradients presented in Proposition \ref{grad-direc} and \ref{grad-sdual}.
We observed that while both critics can be trained to the theoretical optimum, the latter one can always enjoy a faster convergence compared with the former one. An empirical example is visualized in Fig.~\ref{fig:1d_gan}. As shown here, with the same initial states and hyper-parameters, the critic of SWGAN holds a faster and smoother convergence towards the optimal state. This meaningful observation verifies our conjecture that a larger function space of the critic would benefit the training. 


\begin{figure}[!b] 
	\centering
	\begin{subfigure}{0.25\linewidth}
		\centering
		\includegraphics[width=0.99\columnwidth]{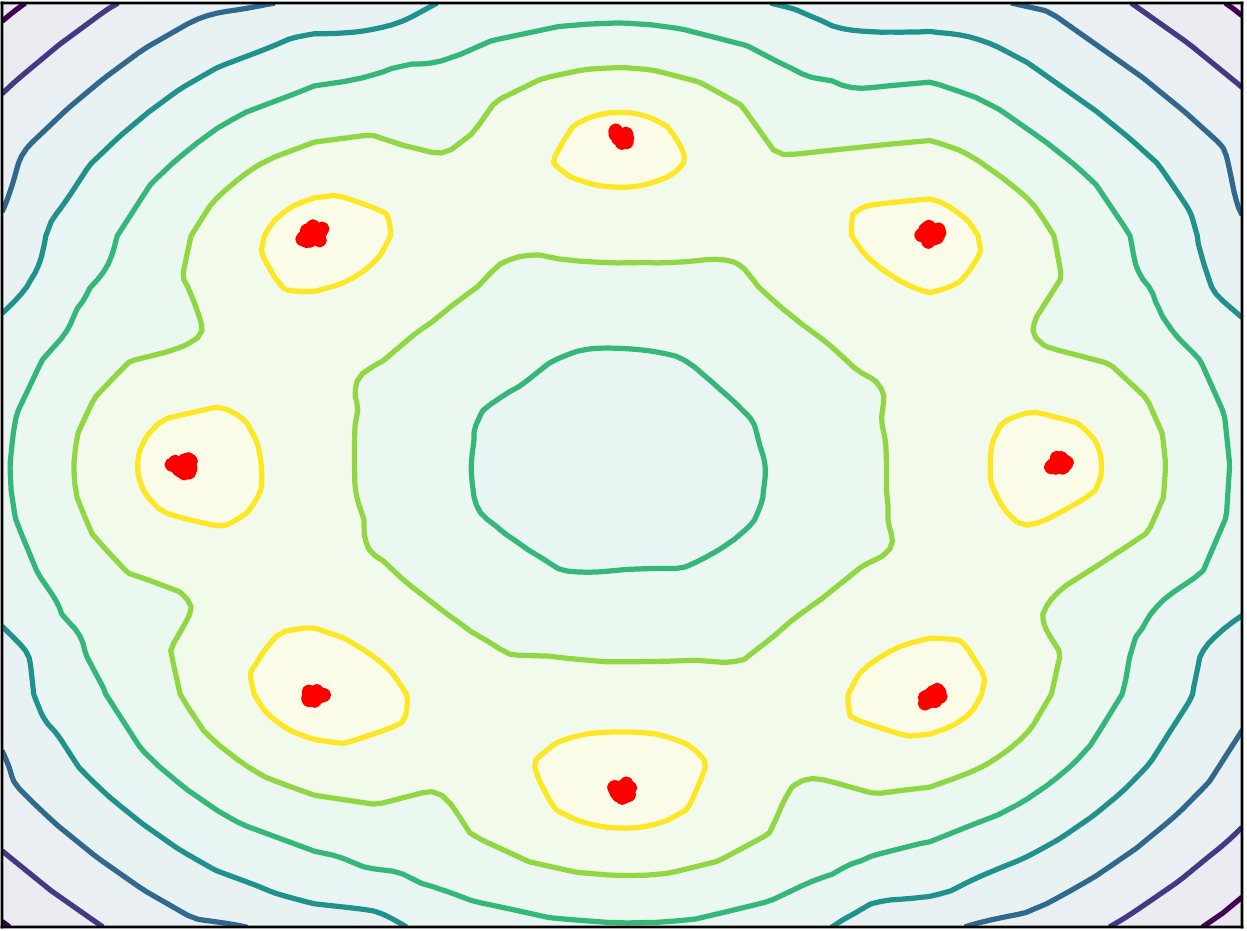}
		\caption{8 Gaussians}
		\label{fig:swgan-al}
	\end{subfigure}
	\hspace{2pt}
	\begin{subfigure}{0.25\linewidth}
		\centering
		\includegraphics[width=0.99\columnwidth]{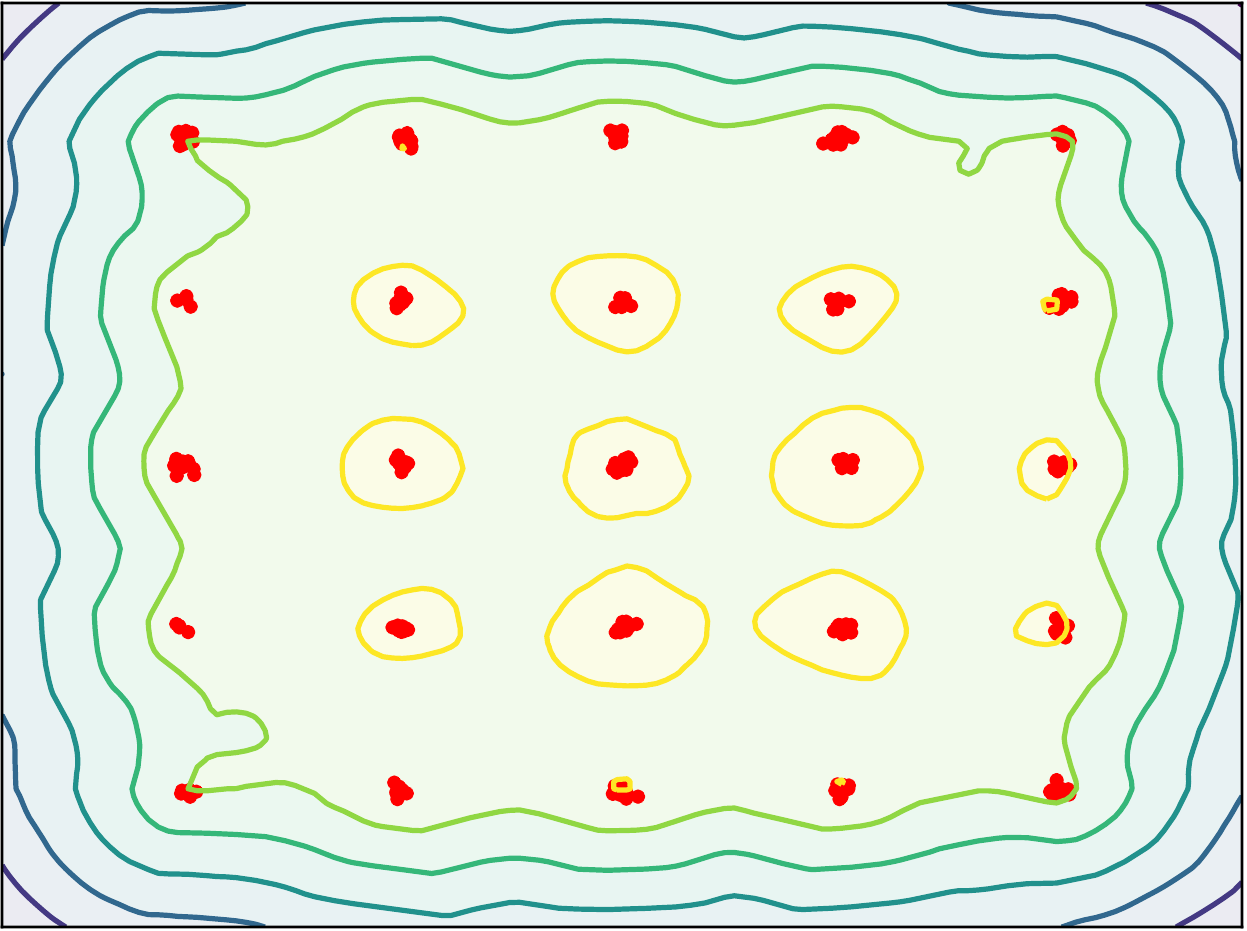}
		\caption{25 Gaussians}
		\label{fig:wgan-al}
	\end{subfigure}	
	\hspace{2pt}
	\begin{subfigure}{0.25\linewidth}
		\centering
		\includegraphics[width=0.99\columnwidth]{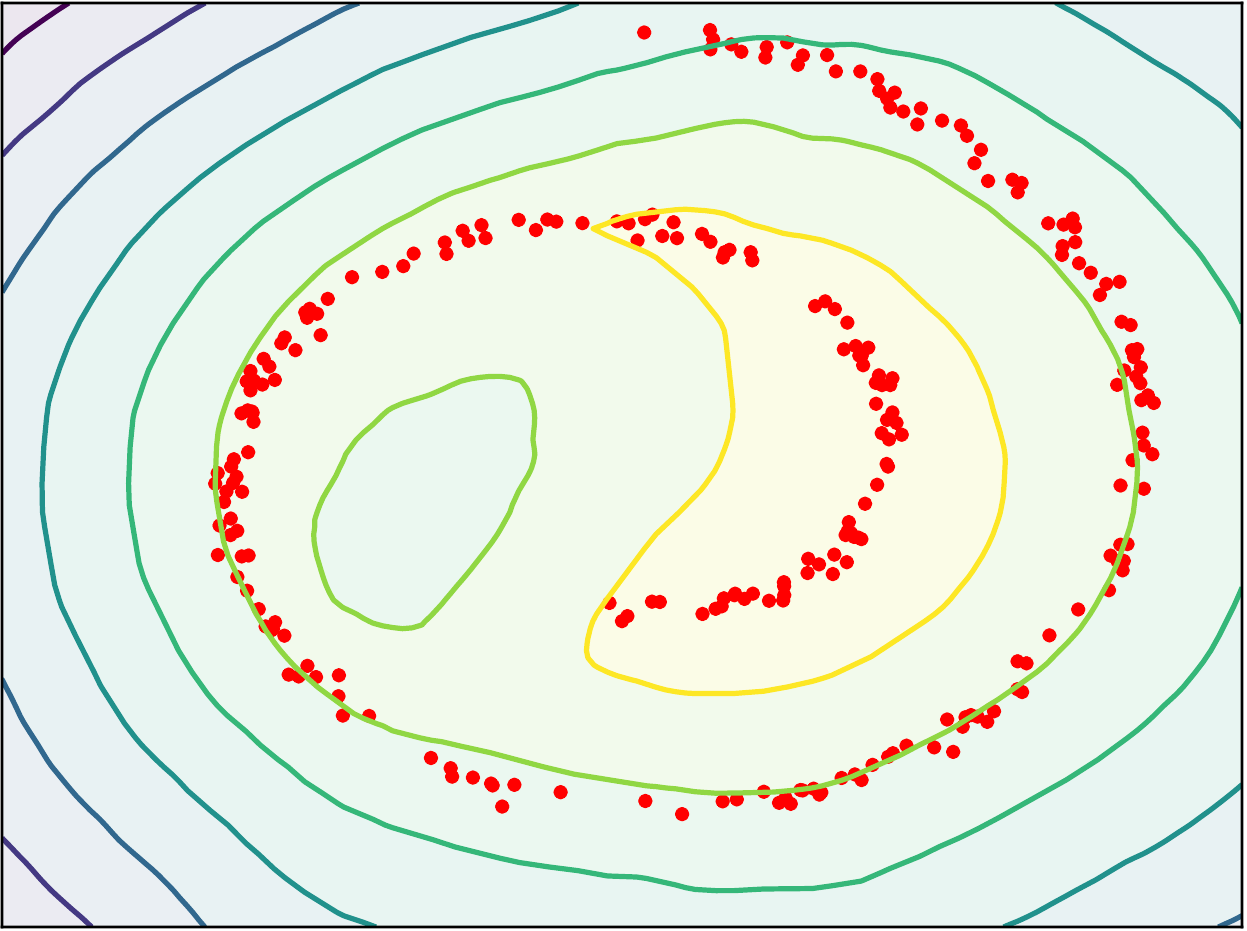}
		\caption{Swiss Roll}
		\label{fig:wgan-gp}
	\end{subfigure}	
	\vspace{-5pt}
	\caption{Level sets of SWGANs critic. Yellow corresponds to high values and purple to low. The training samples are indicated in red and the generated distribution is fixed at the real data plus Gaussian noise.}
	\label{fig:toy}
\end{figure}

\textbf{Level sets of the critic}
In this section we give another 2D level sets visualization. As analyzed in Section~\ref{section-sdual}, Sobolev constraint in SWGAN is a generalization of Lipschitz constraint. Therefore, theoretically SWGAN critic should also be capable of modeling more challenging real and fake data and providing meaningful gradients. To demonstrate this, we train SWGAN critics to optimality on several toy distributions. The value surfaces of the critics are plotted in Figure~\ref{fig:toy}, which shows good fitness of the distribution following our theoretical analysis. 

\subsection{Real-world image generation}
\label{image_synthesis}

\subsubsection{Experimental setup}

\begin{table*}[!t] 
	\vspace{-5pt}
	\centering
	\footnotesize
	\begin{threeparttable}[b]
	\begin{tabular}{|c|c | c|c | c|}
		\hline \rule{0pt}{2.5ex}
		\multirow{2}{*}{GANs} & \multicolumn{2}{c|}{CIFAR-10} & \multicolumn{2}{c|}{Tiny-ImageNet}  \\
		\cline{2-5} \rule{0pt}{2.4ex} & IS & FID & IS & FID  \\
		\hline
		\hline \rule{0pt}{2.2ex}
		WGAN-GP$^*$ & 7.85$\pm$.07 & 18.21$\pm$.12 & 8.17$\pm$.03 & 18.70$\pm$.05 \\
		WGAN-GP with $m=8$ & 7.88$\pm$.09 & 18.08$\pm$.22 & 8.17$\pm$.04 & 18.69$\pm$.10 \\
		WGAN-AL & 7.79$\pm$.09 & 17.86$\pm$.16 & 8.26$\pm$.03 & 18.70$\pm$.06 \\
		WGAN-AL with $m=8$ & 7.89$\pm$.09 & 17.52$\pm$.27 & 8.31$\pm$.02 & 18.61$\pm$.09 \\
		\hline \rule{0pt}{2.6ex}
		SGAN$^*$, $\mu=\frac{P_r+P_g}{2}$ & 7.81$\pm$.11 & 17.89$\pm$.27 & 8.30$\pm$.04 & 18.90$\pm$.04 \\
		SGAN$^*$, $\mu=\mu_{GP}$ & 7.83$\pm$.10 & 18.03$\pm$.24 & 8.31$\pm$.03 & 18.90$\pm$.08 \\
		SGAN, $\mu=\mu_{GP}$ with $m=8$ & 7.86$\pm$.09 & 17.74$\pm$.24 & 8.33$\pm$.03 & 18.75$\pm$.07 \\
		\hline
		\hline \rule{0pt}{2.3ex}
        SWGAN-GP & \textbf{7.98$\pm$.08} & 17.50$\pm$.19 & 8.38$\pm$.03 & 18.50$\pm$.03\\
        SWGAN-AL & 7.93$\pm$.09 & \textbf{16.75$\pm$.24} & \textbf{8.41$\pm$.03} & \textbf{18.32$\pm$.05} \\
		\hline
	\end{tabular}
	\begin{tablenotes}
		\item[*] denotes the vanilla version of our baselines.
	\end{tablenotes}
	\end{threeparttable}
	\captionof{table}{Performance of GANs on CIFAR-10 and Tiny-ImageNet.}
	\vspace{-5pt}
	\label{table:performance}
\end{table*}

\textbf{Controlled variables.}
To make the comparisons more convincing, we also include extended versions of the existing GAN models to control the contrastive variables. 
The controlled variables includes: 
\begin{itemize}
	\vspace{-3pt}
	\item Sampling size. In SWGAN-AL 
	we need to sample $m$ points on each interpolation line between $P_r$ and $P_g$, while in WGAN-GP \cite{wgan-gp} and SGAN \cite{sgan} only one point is sampled. To yield a more fair comparison, we perform additional experiments of WGAN and SGAN with the sampling size equal to $m$.
	\vspace{-3pt}
	\item Optimization method. In our baseline WGAN-GP \cite{wgan-gp}, the restriction is imposed by Penalty Method (PM). By contrast, SGAN and SWGAN-AL use Augmented Lagrangian Method (ALM). 
	ALM is a more advanced algorithm than PM for strictly imposing the constraint. 
	To see the practical difference, we add experiment settings of SWGAN with penalty regularization term (named SWGAN-GP). Formally, the penalty can be written as:
	\begin{equation}
	\begin{aligned}
	\label{eq:gp-loss}
		\mathcal{L}_{gp}&(w,\theta) = - \lambda \; \mathbb{E}_{x_i \sim P_r} \mathbb{E}_{x_j \sim P_g} \Omega_{ij} ^2 (D_w,G_\theta) , 
	\end{aligned}
	\end{equation}
	where $\lambda$ is the gradient penalty coefficient. 
	$\mathcal{L}_{gp}$ is the alternative term of the ALM penalty $\mathcal{L}_{al}$ in Eq.~\eqref{loss-alm} for the training of SWGAN-GP. 
\end{itemize}

\textbf{Baselines.} For comparison, we also evaluated the WGAN-GP \cite{wgan-gp} and Sobolev GAN (SGAN) \cite{sgan} with different sampling sizes and penalty methods. The choice of baselines is due to their close relation to SWGAN as analyzed in Section \ref{wgan-sec}. We omit other previous methods since as a representative of state-of-the-art GAN model, WGAN-GP has been shown to rival or outperform a number of former methods, such as the original GAN \cite{gan}, Energy-based generative adversarial network \cite{zhao2016energy}, the original WGAN with weight clipping \cite{wgan}, Least Squares GAN \cite{mao2017least},  Boundary equilibrium GAN \cite{berthelot2017began} and GAN with denoising feature matching \cite{warde2016improving}.

\textbf{Evaluation metrics.}
Since GAN lacks the capacity to perform reliable likelihood estimations \cite{theis2015note}, we instead concentrate on evaluating the quality of generated images. We choose to compare the maximal Frechet Inception Distances (FID)~\cite{two_time_scale_gan} 
and Inception Scores ~\cite{improved_gan} 
reached during training iterations, both computed from 50K samples. A high image quality corresponds to high Inception and low FID scores. 
Specifally, IS is defined as $\exp(\mathbb{E}_x\operatorname{KL}(p(y|x)||p(y)))$, where $p(y|x)$ is the distribution of label $y$ conditioned on generated data $x$, and $p(y)$ is the marginal distribution. 
IS combines both the confidence of the class predictions for each synthetic images (quality) and the integral of the marginal probability of the predicted classes (diversity).
The classification probabilities were estimated by the Inception model \cite{improved_gan}, a classifier pre-trained upon the ImageNet dataset \cite{deng2009imagenet}. 
However, in practice we note that IS is hard to detect the mode collapse problems. 
FID use the same Inception model to capture computer-vision-specific features of a collection of real and generated images, and then calculate the Frechet distance 
(also called $Wasserstein$-$2$ distance) 
\cite{aronov2006frechet} between two activation distributions. 
The intuition of IS is that high-quality images should lead to high confidence in classification, while FID is aiming to measure the computer-vision-specific similarity of generated images to real ones through Frechet distance (\textit{a.k.a.},~\textit{Wasserstein-}$2$ distance)~\cite{aronov2006frechet}.

\textbf{Data.}
We test different GANs on CIFAR-10~\cite{krizhevsky2009learning} 
and Tiny-ImageNet~\cite{deng2009imagenet}
, which are standard datasets widely used in GANs literatures. Both datasets consist of
tens of thousands of real-world color images with class labels. 

\textbf{Network architecture.}
For all experimental settings, we follow WGAN-GP \cite{wgan-gp} and adopt the same Residual Network (ResNet)~\cite{he2016deep} 
structures and hyperparameters. Details of the architecture are provided in Appendix.

\textbf{Other implementation details.}
For SWGAN metaparameter, we choose $8$ as the sample size $m$. Adam optimizer ~\cite{kingma2014adam} 
is set with learning rate decaying from $2\cdot 10^{-4}$ to $0$ over 100K iterations with $\beta_1=0,\beta_2=0.9$. We used $5$ critic updates per generator update, and the batch size used was $64$.

\begin{figure*}[!t]
	\centering
	\begin{minipage}{0.49\linewidth}
		\centering
    	\includegraphics[width=1.\linewidth]{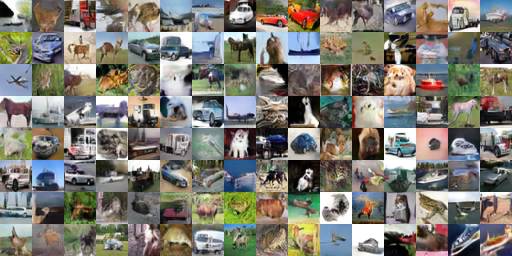}
    	\vspace{-5pt}
    	\caption{Generated CIFAR-10 samples.} 
    	\label{fig:exp:swgan-cifar}
    \end{minipage}
    \begin{minipage}{0.49\linewidth}
        \centering
    	\includegraphics[width=1.\linewidth]{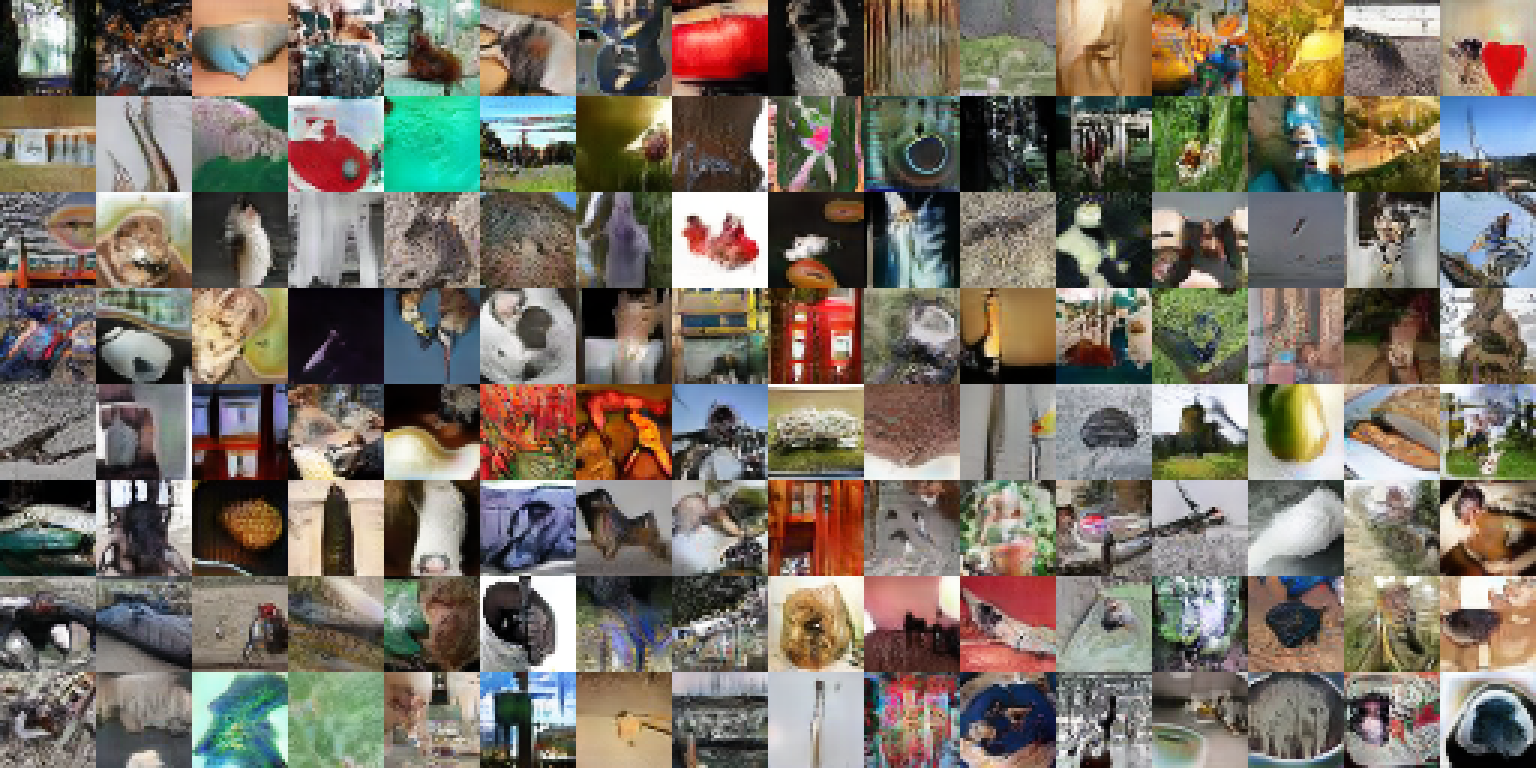}
    	\vspace{-5pt}
    	\caption{Generated Tiny-ImageNet samples.}
    	\label{fig:exp:swgan-tiny}
    \end{minipage}
\end{figure*}

\begin{figure*}[!b] 
    \begin{minipage}{0.47\linewidth}
        \centering
    	\includegraphics[width=.95\linewidth]{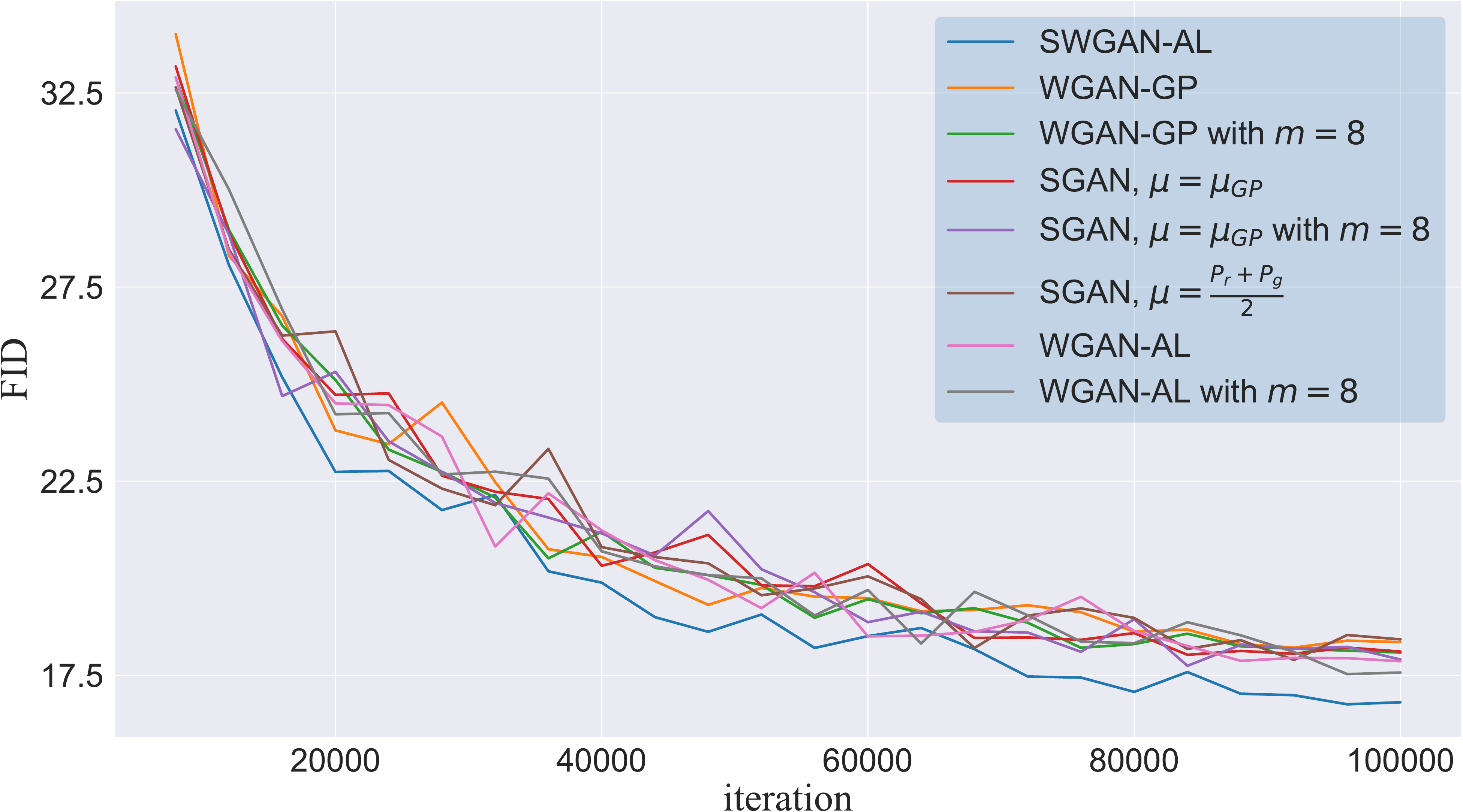}
    	\caption{Training curves on CIFAR-10.} 
    	\label{fig:curve}
    \end{minipage}
    \begin{minipage}{0.53\linewidth}
        \centering
    	\hspace{-5pt}
    	\includegraphics[width=0.99\linewidth]{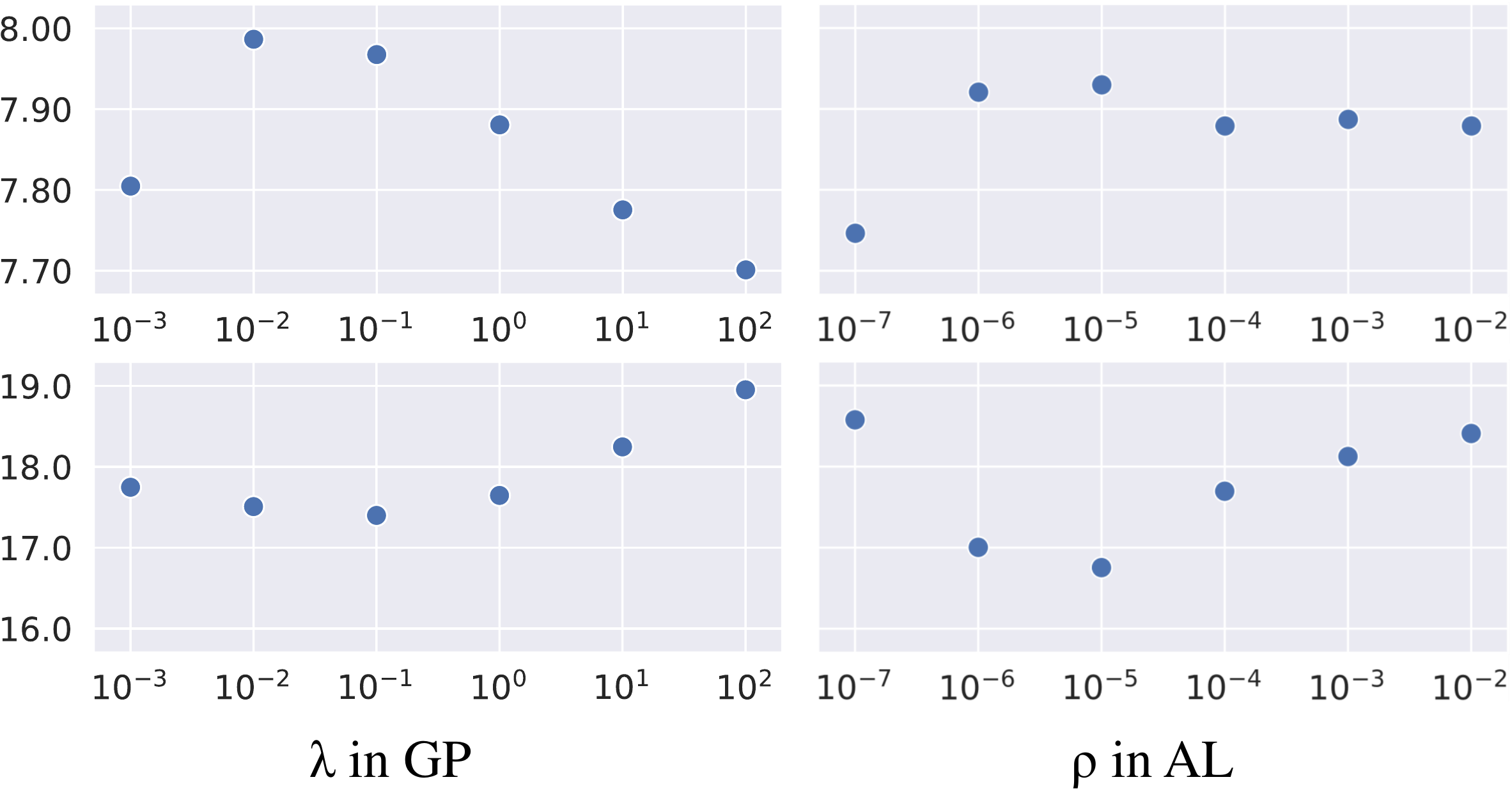}
    	\caption{The comparison of SWGANs with different regularization terms and parameters. Top: Inception  scores. Bottom: Frechet Inception Distances.} 
    	\label{fig:sensitive}
    \end{minipage}
\end{figure*}

\subsubsection{Results}
We here also introduce WGAN with Augmented Lagrangian Method (WGAN-AL) for further comparison, which is similar to SGAN \cite{sgan}. Scores in terms of FID and IS on CIFAR-10 and Tiny-ImageNet are reported in Table~\ref{table:performance}.
Some representative samples from the resulting generator of SWGAN are provided in Fig.~\ref{fig:exp:swgan-cifar} and Fig.~\ref{fig:exp:swgan-tiny}. 
Some representative samples from the resulting generator of SWGAN are provided in Appendix. In experiments, we note that IS is remarkably unstable during training and among different initializations, while FID is fairly stable. So we plot the training curves in terms of FID on CIFAR-10 in Figure~\ref{fig:curve} to show the training process of different GANs.

From Table~\ref{table:performance} and Figure~\ref{fig:curve}, we can see that SWGANs generally work better than the baseline models. The experimental results also show that WGAN and SGAN tend to have slightly better performance when using ALM or sample more interpolation points. However, compared with SWGAN-AL and SWGAN-GP, the performances in these cases are still not competitive enough. This indicates that the larger sampling size and ALM optimization algorithm are not the key elements for the better performance of SWGAN, \textit{i.e.}, these results evidence that \textit{it is the relaxed constraint in Sobolev duality that leads to the improvement}, which is in accordance with our motivation that a looser constraint would simplify the constrained optimization problem and lead to a stronger GAN model. 

We further test SWGAN with different regularization terms and parameters on CIFAR-10. 
The scores are shown in Figure~\ref{fig:sensitive}. 
As shown in Figure~\ref{fig:sensitive}, generally ALM is a better choice when considering FID, while GP is better for IS. 
A meaningful observation is that SWGAN is not sensitive to different values of penalty weights $\rho$ and $\lambda$. By contrast, a previous large scale study reported that the performance of WGAN-GP holds strong dependence on the penalty weight $\lambda$ (see Figure 8 and 9 in \cite{lucic2017gans}). This phenomenon demonstrates a more smooth and stable convergence and well-behaved critic performance throughout the whole training process of SWGAN. 
\section{Local Convergence Analysis}
\label{sec:suboptimal}

\begin{figure}[!t] 
	\centering
	\begin{subfigure}{0.32\linewidth}
		\centering
		\includegraphics[width=0.99\columnwidth]{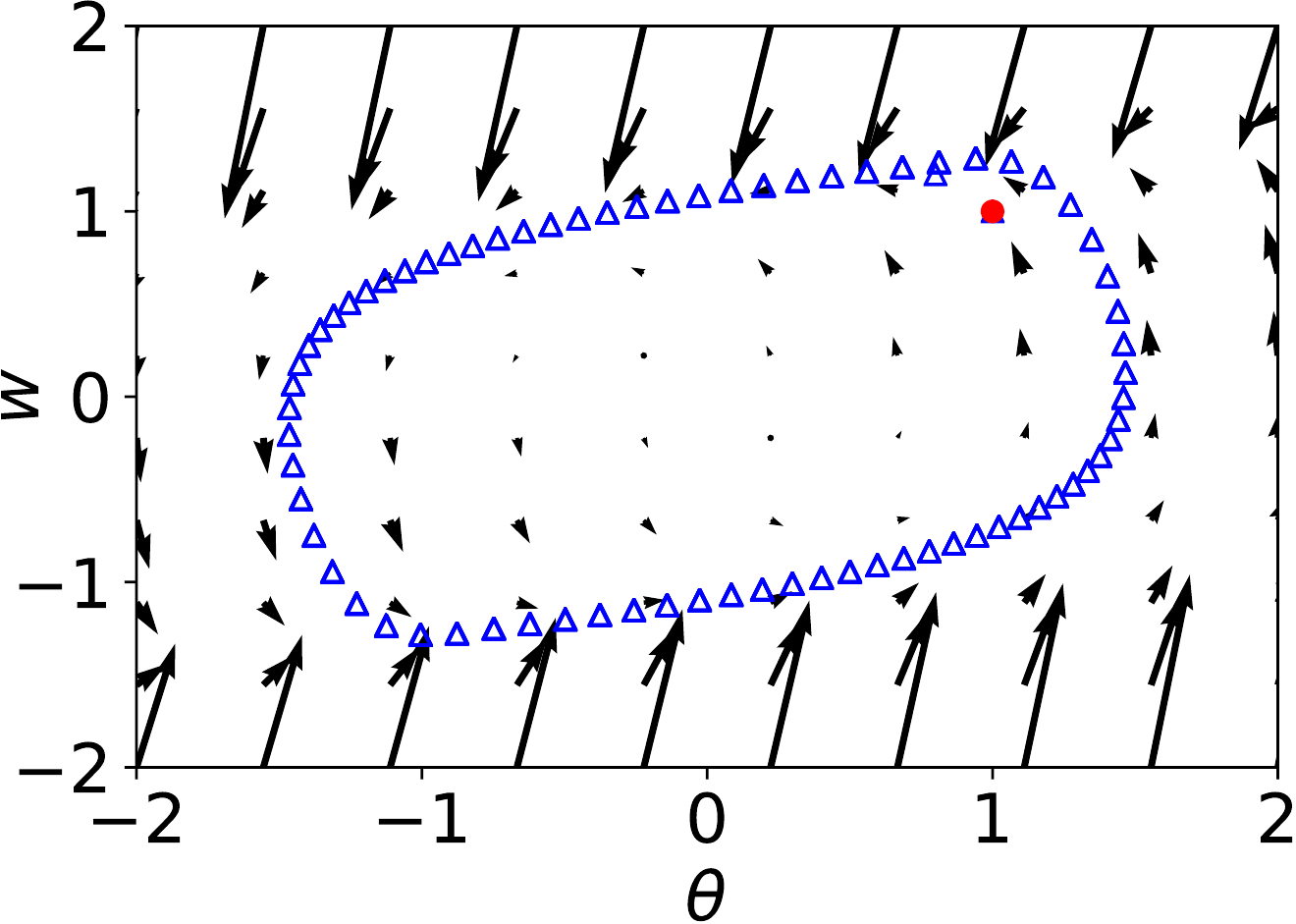}
		\vspace{-10pt}
		\caption{SimGD}
		\vspace{+5pt}
		\label{fig:gp-sgd}
	\end{subfigure}
	\begin{subfigure}{0.32\linewidth}
		\centering
		\includegraphics[width=0.99\columnwidth]{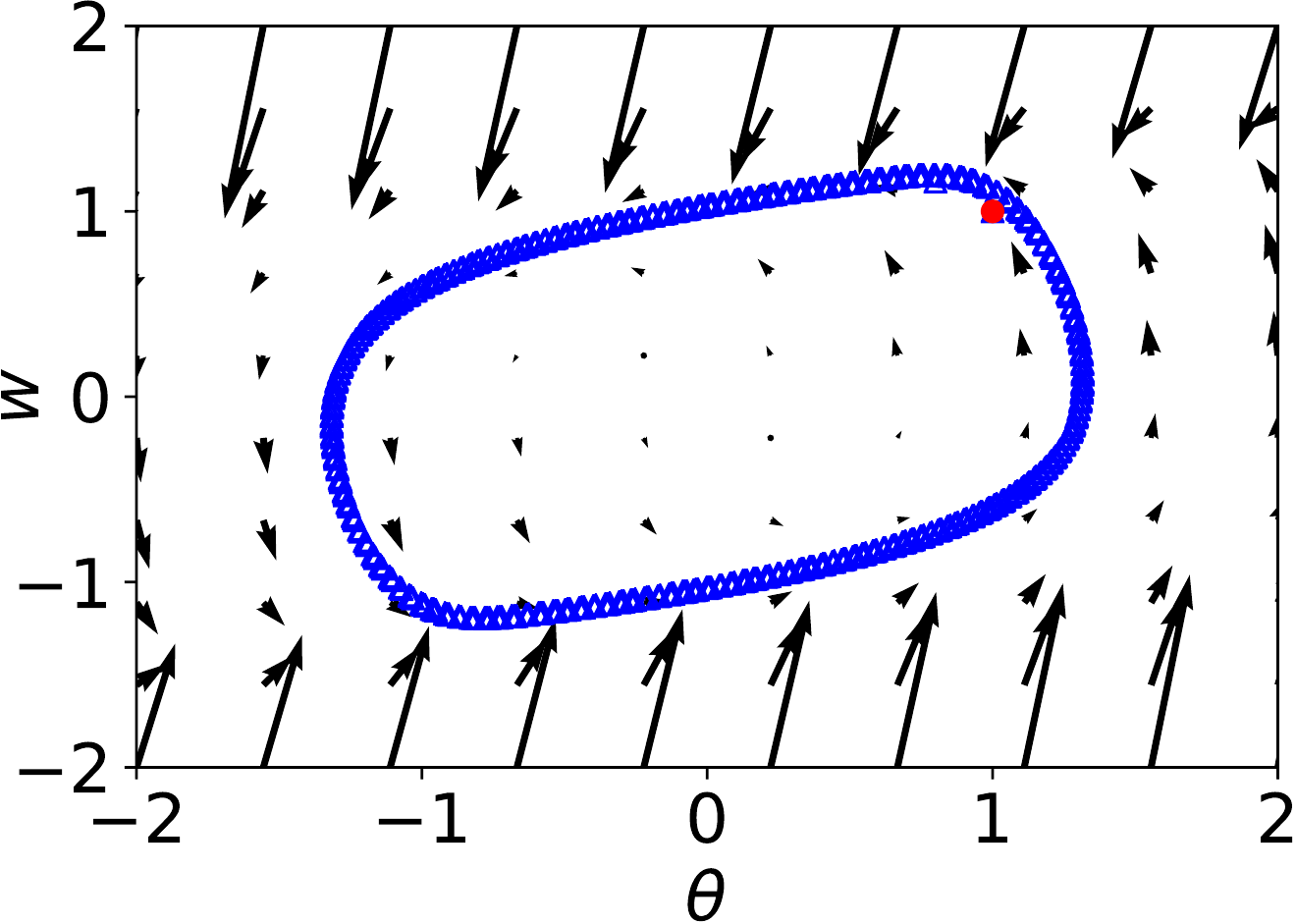}
		\vspace{-10pt}
		\caption{AltGD($n_d$=$1$)}
		\vspace{+5pt}
		\label{fig:gp-agd1}
	\end{subfigure}	
	\begin{subfigure}{0.32\linewidth}
		\centering
		\includegraphics[width=0.99\columnwidth]{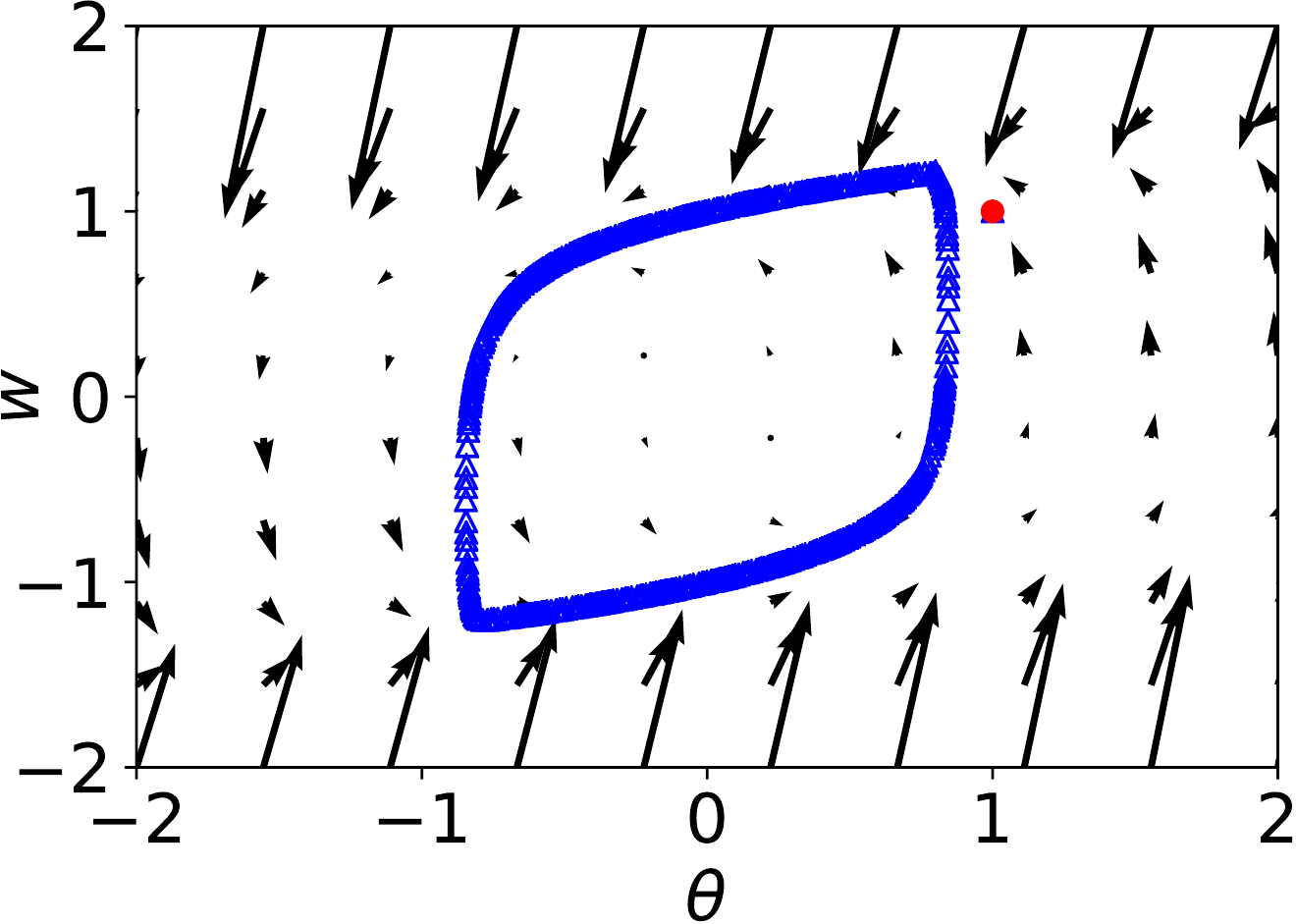}
		\vspace{-10pt}
		\caption{AltGD($n_d$=$5$)}
		\vspace{+5pt}
		\label{fig:gp-agd5}
	\end{subfigure}	
%
	\begin{subfigure}{0.32\linewidth}
		\centering
		\includegraphics[width=0.99\columnwidth]{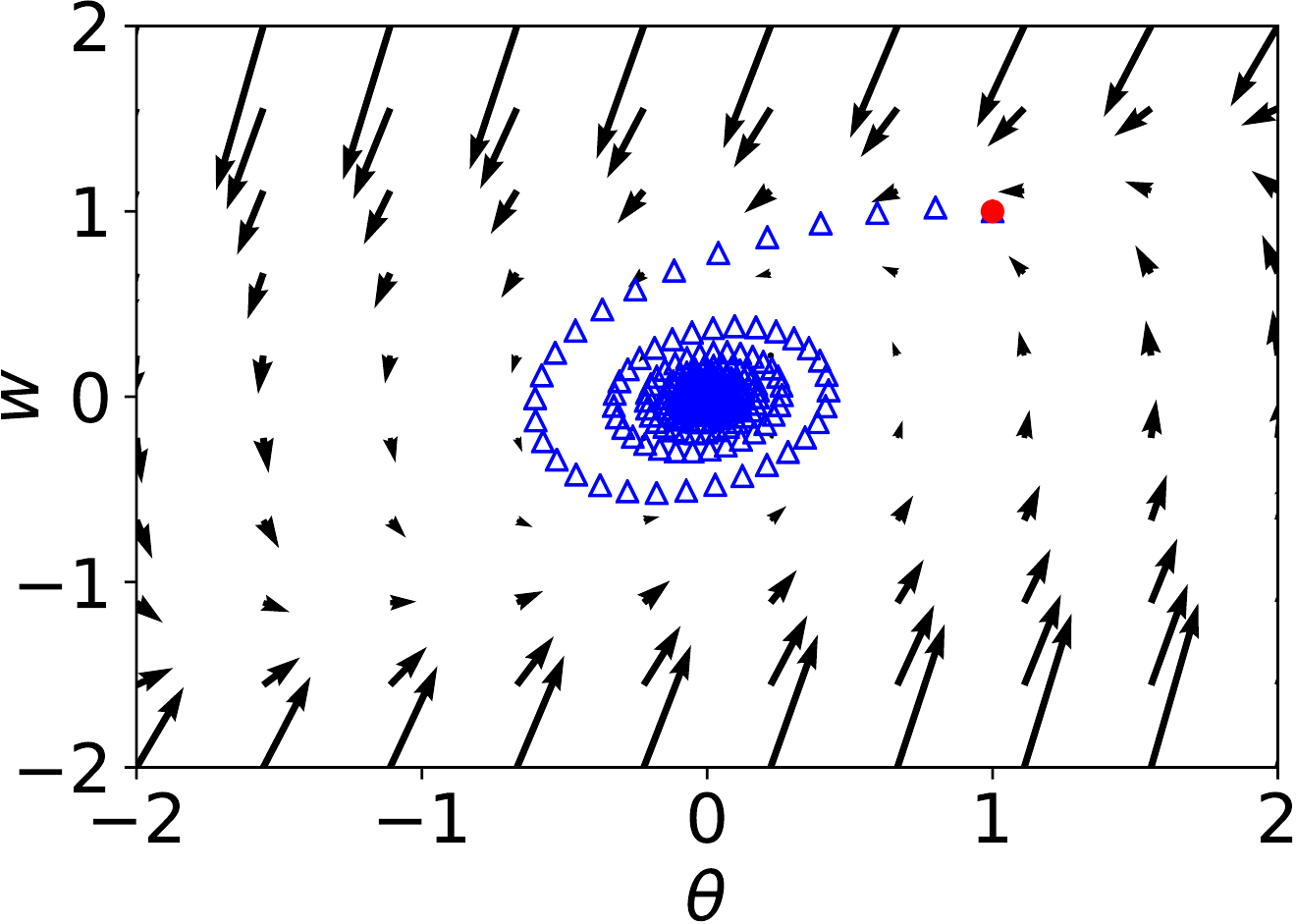}
		\vspace{-10pt}
		\caption{SimGD}
		\label{fig:cp-sgd}
	\end{subfigure}
	\begin{subfigure}{0.32\linewidth}
		\centering
		\includegraphics[width=0.99\columnwidth]{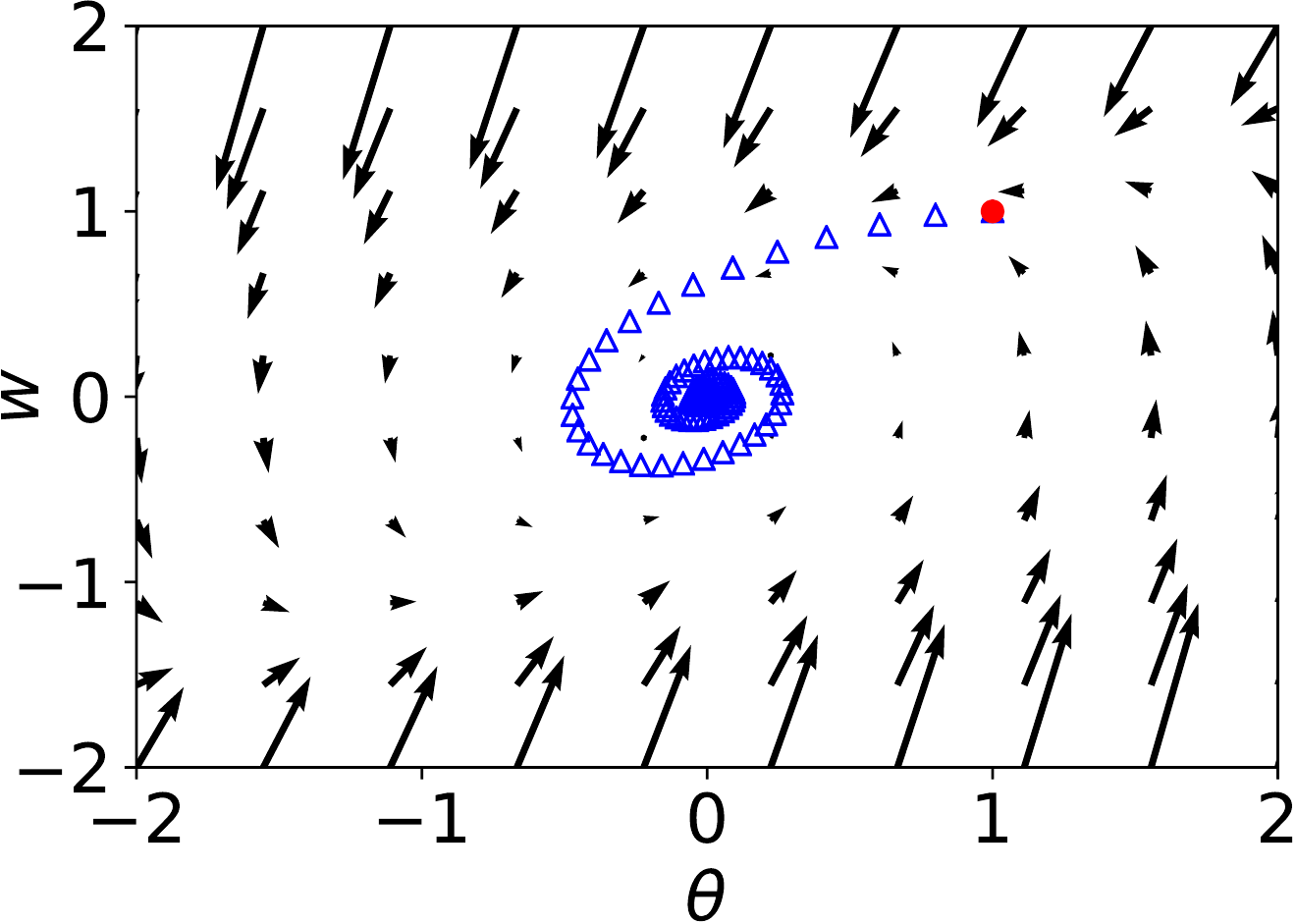}
		\vspace{-10pt}
		\caption{AltGD($n_d$=$1$)}
		\label{fig:cp-agd1}
	\end{subfigure}	
	\begin{subfigure}{0.32\linewidth}
		\centering
		\includegraphics[width=0.99\columnwidth]{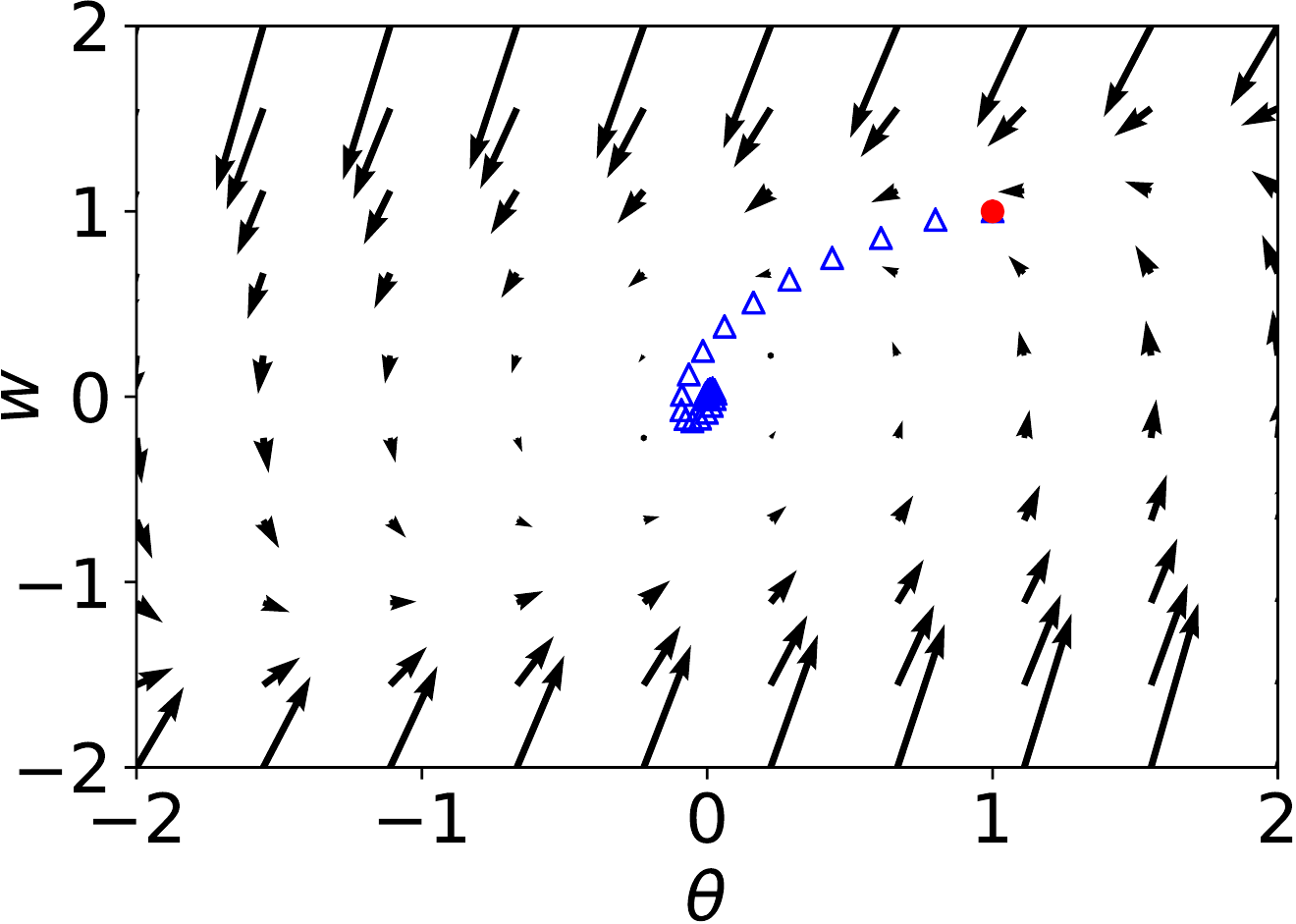}
		\vspace{-10pt}
		\caption{AltGD($n_d$=$5$)}
		\label{fig:cp-agd5}
	\end{subfigure}	
	\caption{Training behavior of SWGANs with different gradient penalty and gradient desent method. The starting iterate is marked in red. Upper: SWGAN-GP. Bottom: SWGAN-CP. First column: update the generator and critic simultaneously. Second and third column: $n_d$ critic updates per generator update.}
	\label{fig:vector-swgan-cp}
\end{figure}

Some literature pay attention to the sub-optimal convergence problem of GANs \cite{mescheder2017numerics,mescheder2018training,arora2017generalization,farnia2018convex},
which is an important direction that deserves further investigation. Moreover, \cite{mescheder2018training} recently shows that in some toy case that distributions that are not absolutely continuous, WGAN based models are not always convergent. 

We employ the same prototypical counterexample to further analyze the convergence of SWGAN: the real data distribution $P_r$ is defined as a Dirac-distribution concentrated at 0; generator distribution is set as $P_g=\delta_\theta$ and the linear critic is set as $D_w(x) = w \cdot x$. We also follow them to focus on investigating the gradient penalty version (SWGAN-GP). In the toy case, the regularizer term $\mathcal{L}_{gp}$ in Eq.~\eqref{eq:gp-loss} can be simplified to:
\begin{equation}
\label{eq:toy-case-gp-loss}
	\mathcal{L}_{gp}(w)= -\lambda \left(w^2-g_{0}\right)^{2},
\end{equation}
where $g_0$ is the gradient penalty target that set as $1$ in SWGAN-GP. Following the previous work, we set the object $\mathcal{L}(w,\theta)$ as $\mathcal{L}_S(w,\theta)+\mathcal{L}_{gp}(w)$ and then study the sub-optimal convergence problem through the \textit{gradient vector field}:
\begin{equation}
\label{eq:gp-field}
	v_{gp}(w,\theta) :=\left(\hspace{-5pt}\begin{array}{c}{-\nabla_{\theta} \mathcal{L}(w, \theta)} \\ {\nabla_{w} \mathcal{L}(w,\theta)}\end{array}\hspace{-5pt}\right) = \left(\hspace{-5pt}\begin{array}{c}{-w} \\ {\theta- 4w \lambda \left(w^2-g_{0}\right)}\end{array}\hspace{-5pt}\right).
\end{equation}

In practice, GANs are usually trained through either Simultaneous or Alternating Gradient Descent (SimGD and AltGD), which means updating the parameters $(w,\theta)$ according to the gradient field simultaneously or alternately. \cite{mescheder2018training} proves that the unique equilibrium point in the above defined case is $w=\theta=0$. However, in Eq.~\eqref{eq:gp-field} the gradient field shows a discontinuity at the equilibrium point,
which indicates that similar to WGAN-GP, SWGAN-GP also cannot converge locally in this case (Fig.~\ref{fig:gp-sgd},~\ref{fig:gp-agd1},~and~\ref{fig:gp-agd5}). To bypass this minor sub-optimal problem, we design a zero-centered gradient penalty (SWGAN-CP) by simply setting the $g_0$ in Eq.~\eqref{eq:toy-case-gp-loss} as $0$ and adding another penalty on the squared norm of the
gradients (Fig.~\ref{fig:cp-sgd},~\ref{fig:cp-agd1},~and~\ref{fig:cp-agd5}). The proof for the local convergence of SWGAN-CP is provided in Appendix.

However, as acknowledged by \cite{mescheder2018training}, despite these minor ill-behaviors, WGAN has been successfully applied in practice. Furthermore, in our real-world experiments, we also noticed that there are few positive improvements when taking zero-centered penalty as the regularization term. 

\section{Conclusion}
In this paper, we proposed a new dual form of Wasserstein distance with the Lipschitz constraint relaxed and demonstrate that it is still capable of eliminating the training instability issues. This new dual form leads to a generalized WGAN model. We built Sobolev Wasserstein GAN based on the proposed duality and provided empirical evidence that our GAN model outperforms the previous approaches, which either impose the strong Lipschitz penalty or cannot theoretically guarantee the convergence.

This work was motivated by the intuition that with a less restricted function space, the critic would be easier to be trained to optimum, thus benefiting the training of GANs. To the best of our knowledge, Sobolev Wasserstein GAN is the GAN model with the most relaxed restriction that can still avoid the training instability problem. In the future, we hope that practitioners can take a step back and investigate whether we can further relax the constraint imposed on the function space of the critic and what the minimal requirement for the convergence guarantee could be. 
\newpage
\bibliographystyle{named}
\bibliography{neurips_2020}

\begin{thebibliography}{}

\bibitem[\protect\citeauthoryear{Adams and Fournier}{2003}]{2003sobolev}
Robert~A Adams and John~JF Fournier.
\newblock {\em Sobolev spaces}, volume 140.
\newblock Elsevier, 2003.

\bibitem[\protect\citeauthoryear{Adler and Lunz}{2018}]{adler2018banach}
Jonas Adler and Sebastian Lunz.
\newblock Banach wasserstein gan.
\newblock In {\em NIPS}, 2018.

\bibitem[\protect\citeauthoryear{Arjovsky and
  Bottou}{2017}]{principled_methods}
Martin Arjovsky and L{\'e}on Bottou.
\newblock Towards principled methods for training generative adversarial
  networks.
\newblock In {\em ICLR}, 2017.

\bibitem[\protect\citeauthoryear{Arjovsky \bgroup \em et al.\egroup
  }{2017}]{wgan}
Martin Arjovsky, Soumith Chintala, and L{\'e}on Bottou.
\newblock Wasserstein generative adversarial networks.
\newblock In {\em ICML}, 2017.

\bibitem[\protect\citeauthoryear{Aronov \bgroup \em et al.\egroup
  }{2006}]{aronov2006frechet}
Boris Aronov, Sariel Har-Peled, Christian Knauer, Yusu Wang, and Carola Wenk.
\newblock Fr{\'e}chet distance for curves, revisited.
\newblock In {\em European Symposium on Algorithms}. Springer, 2006.

\bibitem[\protect\citeauthoryear{Arora \bgroup \em et al.\egroup
  }{2017}]{arora2017generalization}
Sanjeev Arora, Rong Ge, Yingyu Liang, Tengyu Ma, and Yi~Zhang.
\newblock Generalization and equilibrium in generative adversarial nets (gans).
\newblock {\em arXiv preprint arXiv:1703.00573}, 2017.

\bibitem[\protect\citeauthoryear{Bellemare \bgroup \em et al.\egroup
  }{2017}]{bellemare2017cramer}
Marc~G Bellemare, Ivo Danihelka, Will Dabney, Shakir Mohamed, Balaji
  Lakshminarayanan, Stephan Hoyer, and R{\'e}mi Munos.
\newblock The cramer distance as a solution to biased wasserstein gradients.
\newblock {\em arXiv preprint arXiv:1705.10743}, 2017.

\bibitem[\protect\citeauthoryear{Berthelot \bgroup \em et al.\egroup
  }{2017}]{berthelot2017began}
David Berthelot, Thomas Schumm, and Luke Metz.
\newblock Began: Boundary equilibrium generative adversarial networks.
\newblock {\em arXiv preprint arXiv:1703.10717}, 2017.

\bibitem[\protect\citeauthoryear{Che \bgroup \em et al.\egroup
  }{2016}]{mode_gan}
Tong Che, Yanran Li, Athul~Paul Jacob, Yoshua Bengio, and Wenjie Li.
\newblock Mode regularized generative adversarial networks.
\newblock {\em arXiv preprint arXiv:1612.02136}, 2016.

\bibitem[\protect\citeauthoryear{Deng \bgroup \em et al.\egroup
  }{2009}]{deng2009imagenet}
Jia Deng, Wei Dong, Richard Socher, Li-Jia Li, Kai Li, and Li~Fei-Fei.
\newblock Imagenet: A large-scale hierarchical image database.
\newblock In {\em CVPR}. Ieee, 2009.

\bibitem[\protect\citeauthoryear{Farnia and Tse}{2018}]{farnia2018convex}
Farzan Farnia and David Tse.
\newblock A convex duality framework for gans.
\newblock In {\em NIPS}, 2018.

\bibitem[\protect\citeauthoryear{Fedus \bgroup \em et al.\egroup
  }{2017}]{fedus2017many}
William Fedus, Mihaela Rosca, Balaji Lakshminarayanan, Andrew~M Dai, Shakir
  Mohamed, and Ian Goodfellow.
\newblock Many paths to equilibrium: Gans do not need to decrease a divergence
  at every step.
\newblock {\em arXiv preprint arXiv:1710.08446}, 2017.

\bibitem[\protect\citeauthoryear{Goodfellow \bgroup \em et al.\egroup
  }{2014}]{gan}
Ian Goodfellow, Jean Pouget-Abadie, Mehdi Mirza, Bing Xu, David Warde-Farley,
  Sherjil Ozair, Aaron Courville, and Yoshua Bengio.
\newblock Generative adversarial nets.
\newblock In {\em NIPS}, 2014.

\bibitem[\protect\citeauthoryear{Goodfellow}{2016}]{gan-tutorial}
Ian Goodfellow.
\newblock Nips 2016 tutorial: Generative adversarial networks.
\newblock {\em arXiv preprint arXiv:1701.00160}, 2016.

\bibitem[\protect\citeauthoryear{Gulrajani \bgroup \em et al.\egroup
  }{2017}]{wgan-gp}
Ishaan Gulrajani, Faruk Ahmed, Martin Arjovsky, Vincent Dumoulin, and Aaron~C
  Courville.
\newblock Improved training of wasserstein gans.
\newblock In {\em NIPS}, 2017.

\bibitem[\protect\citeauthoryear{He \bgroup \em et al.\egroup
  }{2016}]{he2016deep}
Kaiming He, Xiangyu Zhang, Shaoqing Ren, and Jian Sun.
\newblock Deep residual learning for image recognition.
\newblock In {\em CVPR}, 2016.

\bibitem[\protect\citeauthoryear{Heusel \bgroup \em et al.\egroup
  }{2017a}]{heusel2017gans}
Martin Heusel, Hubert Ramsauer, Thomas Unterthiner, Bernhard Nessler, and Sepp
  Hochreiter.
\newblock Gans trained by a two time-scale update rule converge to a local nash
  equilibrium.
\newblock In {\em NIPS}, 2017.

\bibitem[\protect\citeauthoryear{Heusel \bgroup \em et al.\egroup
  }{2017b}]{two_time_scale_gan}
Martin Heusel, Hubert Ramsauer, Thomas Unterthiner, Bernhard Nessler,
  G{\"u}nter Klambauer, and Sepp Hochreiter.
\newblock Gans trained by a two time-scale update rule converge to a nash
  equilibrium.
\newblock {\em arXiv preprint arXiv:1706.08500}, 2017.

\bibitem[\protect\citeauthoryear{Kingma and Ba}{2014}]{kingma2014adam}
Diederik~P Kingma and Jimmy Ba.
\newblock Adam: A method for stochastic optimization.
\newblock {\em arXiv preprint arXiv:1412.6980}, 2014.

\bibitem[\protect\citeauthoryear{Kodali \bgroup \em et al.\egroup
  }{2017a}]{kodali2017train}
Naveen Kodali, Jacob Abernethy, James Hays, and Zsolt Kira.
\newblock How to train your dragan.
\newblock {\em arXiv preprint arXiv:1705.07215}, 2(4), 2017.

\bibitem[\protect\citeauthoryear{Kodali \bgroup \em et al.\egroup
  }{2017b}]{kodali2017convergence}
Naveen Kodali, Jacob Abernethy, James Hays, and Zsolt Kira.
\newblock On convergence and stability of gans.
\newblock {\em arXiv preprint arXiv:1705.07215}, 2017.

\bibitem[\protect\citeauthoryear{Krizhevsky \bgroup \em et al.\egroup
  }{2009}]{krizhevsky2009learning}
Alex Krizhevsky, Geoffrey Hinton, et~al.
\newblock Learning multiple layers of features from tiny images.
\newblock Technical report, Citeseer, 2009.

\bibitem[\protect\citeauthoryear{Lucic \bgroup \em et al.\egroup
  }{2017}]{lucic2017gans}
Mario Lucic, Karol Kurach, Marcin Michalski, Sylvain Gelly, and Olivier
  Bousquet.
\newblock Are gans created equal? a large-scale study.
\newblock {\em arXiv preprint arXiv:1711.10337}, 2017.

\bibitem[\protect\citeauthoryear{Mao \bgroup \em et al.\egroup
  }{2017}]{mao2017least}
Xudong Mao, Qing Li, Haoran Xie, Raymond~YK Lau, Zhen Wang, and Stephen
  Paul~Smolley.
\newblock Least squares generative adversarial networks.
\newblock In {\em CVPR}, 2017.

\bibitem[\protect\citeauthoryear{Mescheder \bgroup \em et al.\egroup
  }{2017}]{mescheder2017numerics}
Lars Mescheder, Sebastian Nowozin, and Andreas Geiger.
\newblock The numerics of gans.
\newblock In {\em NIPS}, 2017.

\bibitem[\protect\citeauthoryear{Mescheder \bgroup \em et al.\egroup
  }{2018}]{mescheder2018training}
Lars Mescheder, Andreas Geiger, and Sebastian Nowozin.
\newblock Which training methods for gans do actually converge?
\newblock In {\em ICML}, 2018.

\bibitem[\protect\citeauthoryear{Metz \bgroup \em et al.\egroup
  }{2016}]{unrolled_gan}
Luke Metz, Ben Poole, David Pfau, and Jascha Sohl-Dickstein.
\newblock Unrolled generative adversarial networks.
\newblock {\em arXiv preprint arXiv:1611.02163}, 2016.

\bibitem[\protect\citeauthoryear{Miyato \bgroup \em et al.\egroup
  }{2018}]{sngan}
Takeru Miyato, Toshiki Kataoka, Masanori Koyama, and Yuichi Yoshida.
\newblock Spectral normalization for generative adversarial networks.
\newblock In {\em ICLR}, 2018.

\bibitem[\protect\citeauthoryear{Mroueh and Sercu}{2017}]{fishergan}
Youssef Mroueh and Tom Sercu.
\newblock Fisher gan.
\newblock In {\em NIPS}, 2017.

\bibitem[\protect\citeauthoryear{Mroueh \bgroup \em et al.\egroup
  }{2018}]{sgan}
Youssef Mroueh, Chun-Liang Li, Tom Sercu, Anant Raj, and Yu~Cheng.
\newblock Sobolev {GAN}.
\newblock In {\em ICLR}, 2018.

\bibitem[\protect\citeauthoryear{Nagarajan and
  Kolter}{2017}]{nagarajan2017gradient}
Vaishnavh Nagarajan and J~Zico Kolter.
\newblock Gradient descent gan optimization is locally stable.
\newblock In {\em NIPS}, 2017.

\bibitem[\protect\citeauthoryear{Nocedal and Wright}{2006}]{alm}
Jorge Nocedal and Stephen Wright.
\newblock {\em Numerical optimization}.
\newblock Springer Science \& Business Media, 2006.

\bibitem[\protect\citeauthoryear{Odena \bgroup \em et al.\egroup
  }{2018}]{odena2018generator}
Augustus Odena, Jacob Buckman, Catherine Olsson, Tom~B Brown, Christopher Olah,
  Colin Raffel, and Ian Goodfellow.
\newblock Is generator conditioning causally related to gan performance?
\newblock {\em arXiv preprint arXiv:1802.08768}, 2018.

\bibitem[\protect\citeauthoryear{Petzka \bgroup \em et al.\egroup
  }{2018}]{wgan-lp}
Henning Petzka, Asja Fischer, and Denis Lukovnikov.
\newblock On the regularization of wasserstein {GAN}s.
\newblock In {\em ICLR}, 2018.

\bibitem[\protect\citeauthoryear{Salimans \bgroup \em et al.\egroup
  }{2016}]{improved_gan}
Tim Salimans, Ian Goodfellow, Wojciech Zaremba, Vicki Cheung, Alec Radford, and
  Xi~Chen.
\newblock Improved techniques for training gans.
\newblock In {\em NIPS}, 2016.

\bibitem[\protect\citeauthoryear{Theis \bgroup \em et al.\egroup
  }{2015}]{theis2015note}
Lucas Theis, A{\"a}ron van~den Oord, and Matthias Bethge.
\newblock A note on the evaluation of generative models.
\newblock {\em arXiv preprint arXiv:1511.01844}, 2015.

\bibitem[\protect\citeauthoryear{Villani}{2008}]{krdual}
C.~Villani.
\newblock {\em Optimal Transport: Old and New}.
\newblock Grundlehren der mathematischen Wissenschaften. Springer Berlin
  Heidelberg, 2008.

\bibitem[\protect\citeauthoryear{Warde-Farley and
  Bengio}{2016}]{warde2016improving}
David Warde-Farley and Yoshua Bengio.
\newblock Improving generative adversarial networks with denoising feature
  matching.
\newblock 2016.

\bibitem[\protect\citeauthoryear{Yadav \bgroup \em et al.\egroup
  }{2017}]{yadav2017stabilizing}
Abhay Yadav, Sohil Shah, Zheng Xu, David Jacobs, and Tom Goldstein.
\newblock Stabilizing adversarial nets with prediction methods.
\newblock {\em arXiv preprint arXiv:1705.07364}, 2017.

\bibitem[\protect\citeauthoryear{Zhang \bgroup \em et al.\egroup
  }{2018}]{zhang2018self}
Han Zhang, Ian Goodfellow, Dimitris Metaxas, and Augustus Odena.
\newblock Self-attention generative adversarial networks.
\newblock {\em arXiv preprint arXiv:1805.08318}, 2018.

\bibitem[\protect\citeauthoryear{Zhao \bgroup \em et al.\egroup
  }{2016}]{zhao2016energy}
Junbo Zhao, Michael Mathieu, and Yann LeCun.
\newblock Energy-based generative adversarial network.
\newblock {\em arXiv preprint arXiv:1609.03126}, 2016.

\bibitem[\protect\citeauthoryear{Zhou \bgroup \em et al.\egroup }{2019}]{LGAN}
Zhiming Zhou, Jiadong Liang, Yuxuan Song, Lantao Yu, Hongwei Wang, Weinan
  Zhang, Yong Yu, and Zhihua Zhang.
\newblock Lipschitz generative adversarial nets.
\newblock In {\em ICML}, 2019.

\end{thebibliography}
\onecolumn
\newpage
\appendix

\section{Algorithm for SWGAN-AL}
\label{app:sec:algorithm}

\begin{algorithm}[H]
	\caption{Sobolev Wasserstein GAN with Augmented Lagrangian}
	\label{swgan-al}
	\textbf{Input}: $\rho$ penalty weight, $\eta$ learning rate, $k$ number of critic iterations per generator iteration, $n$ batch size, $m$ number of points sampled between each pair of fake and real data
	
	\textbf{Initial}: critic parameters $w$, generator parameters $\theta$, Lagrange multiplier $\alpha$
	\begin{algorithmic}[1] 
		\WHILE{$\theta$ has not converged}
		\FOR{$l=1,...,k$}
		\STATE Sample a minibatch $x_i$, $i=1,...,n$, $x_i\sim P_r$
		\STATE Sample a minibatch $z_i$, $i=1,...,n$, $z_i\sim P_z$
		\STATE $\widetilde{x}_i \leftarrow G_{\theta}(z_i)$, $i=1,...,n$
		\FOR{$i=1,...,n$}
		\STATE Sample a minibatch $\hat{x}_{ij} = t_jx_i + (1-t_j) \widetilde{x}_i$, $j=1,...,m$, $t_j\sim \text{U}(0,1)$
		\STATE $\Omega^{(i)} \leftarrow 1- \frac{1}{m} \sum_{j=1}^{m} \Vert\nabla_xD(\hat{x}_{ij}) \Vert ^2$
		\STATE $s^{(i)} \leftarrow \max\big\{\Omega^{(i)}-\frac{\alpha}{\rho}\,,\,0\big\}$.
		\STATE $\mathcal{L}_{al}^{(i)} \leftarrow \alpha (\Omega^{(i)} - s^{(i)}) - \frac{\rho}{2} (\Omega^{(i)} - s^{(i)})^2$
		\ENDFOR
		\STATE $\mathcal{L} \leftarrow \frac{1}{n} \sum_{i=1}^{n} \big\{ D_w(x_i) - D_w(\widetilde{x}_i) + \mathcal{L}_{al}^{(i)} \big\}$
		\STATE $(g_w,g_{\alpha}) \leftarrow (\nabla_{w} \mathcal{L},\nabla_{\alpha} \mathcal{L}) (w,\theta,\alpha)$
		\STATE $w \leftarrow w + \eta \text{ADAM}(w,g_w)$
		\STATE $\alpha \leftarrow \max \big \{\alpha - \rho g_{\alpha},0 \big \}$
		\ENDFOR
		\STATE Sample a minibatch $z_i$, $i=1,...,n$, $z_i\sim P_z$
		\STATE $g_{\theta} \leftarrow -\nabla_{\theta} \frac{1}{n} \sum_{i=1}^{n} D_w(G_{\theta}(z_i))$
		\STATE $\theta \leftarrow \theta - \eta \text{ADAM} (\theta,g_{\theta})$
		\ENDWHILE
	\end{algorithmic}
\end{algorithm}

\section{Proofs for \emph{Sobolev duality}}

\subsection{Proof of the \emph{Sobolev duality} of Wasserstein Distance}
\label{app:proof:section-sobolev-dual}

We here provide a proof for our new dual form of Wasserstein distance.

The Wasserstein distance is given as follows
\begin{equation}\label{eq_w_primal_app}
W_1(P_r,P_g) =  \inf_{\pi \in \Pi(P_r,P_g)} \, \mathbb{E}_{(x_i,x_j) \sim \pi} \, [\Vert x_i - x_j \Vert],
\end{equation}
where $\Pi(P_r, P_g)$ denotes the set of all probability measures with marginals $P_r$ and $P_g$ on the first and second factors, respectively. The Kantorovich-Rubinstein (KR) dual is written as
\begin{equation}
\begin{aligned}
W_{KR}&(P_r,P_g) = {\sup}_{f} \,\, \Big\{ \mathbb{E}_{x_i \sim P_r} \, [f(x_i)] - \mathbb{E}_{x_j \sim P_g} \, [f(x_j)] \Big\},  \, \\
&\emph{s.t.} \, f(x_i) - f(x_j) \leq \Vert x_i - x_j \Vert, \,\, \forall x_i, \forall x_j.
\end{aligned}
\label{eq_w_dual_form_app1}
\end{equation}
We will prove that Wasserstein distance in its dual form can also be written in \emph{Sobolev duality}:
\begin{equation}
\label{eq_s_dual_form_app2}
\begin{aligned}
    W_S&(P_r,P_g) = {\sup}_{f} \, \Big \{\mathbb{E}_{x \sim P_r} [f(x)] - \mathbb{E}_{x \sim P_g}  [f(x)] \Big \}, \\
    &\emph{s.t.} \, f \in \mathcal{F}_S (\mathcal{X}, \mu ^{x_i,x_j}), \,\, \forall x_i \sim P_r, \forall x_j \sim P_g,
\end{aligned}
\end{equation}
where
\begin{equation}
\label{app:eq:sample-mu}
\mu^{x_i,x_j}(x) = 
\begin{cases}
\dfrac{1}{\Vert x_i - x_j \Vert}, & x = t x_i + (t-1)x_j,
\vspace{+5pt}
\cr
0, &otherwise.
\end{cases}
\end{equation}
which relaxes the Lipschitz constraint in the KR dual form of Wasserstein distance.

\begin{theorem}
Given $W_{KR}(P_r,P_g)=W_1(P_r,P_g)$, we have $W_{KR}(P_r,P_g) = W_{S}(P_r,P_g) = W_1(P_r,P_g) $. 
\end{theorem} 

\begin{proof} $ $\newline 

    (i) As illustrated in Section 3.2.1, Lipschitz continuity is a sufficient condition of the constraint in Sobolev dual. Therefor, for any $f$ that satisfies ``$f(x_i)-f(x_j) \leq \Vert x_i-x_j \Vert, \, \forall x_i, \forall x_j$'', it must satisfy ``$f \in \mathcal{F}_S(\mathcal{X},\mu^{x_i,x_j}), \, \forall x_i \sim P_r, \forall x_j \sim P_g$''. Thus, $W_{KR}(P_r,P_g)\leq W_{S}(P_r,P_g)$.  
    
    
    (ii) Let a curve $\mathcal{C}$ be the straight line from $x_j$ to $x_i$ with its tangent vector at some chosen point defined as $\mathbf{v}$, \textit{i.e.}, $\mathbf{v}=\frac{x_i-x_j}{\Vert x_i-x_j\Vert}$. Let $P_t$ be U$[0, 1]$ and let $t$ be the random variable that follows $P_t$. Let $\mathbf{r}(t)=tx_i+(1-t)x_j$, which is a bijective parametrization of the curve $\mathcal{C}$ such that $\mathbf{r}(0)$ and $\mathbf{r}(1)$ give the endpoints of $\mathcal{C}$. Given $f \in \mathcal{F}_S (\mathcal{X}, \mu ^{x_i,x_j})$, we have:
    
    \begin{align}
        f(x_i)-f(x_j) &= \int_\mathcal{C} \mathbf{v} \cdot \nabla_x f(\mathbf{r}) d \mathbf{r} \nonumber \\ 
        &\leq \int_\mathcal{C} \Vert \nabla_x f(\mathbf{r}) \Vert d \mathbf{r}  \label{inequality-1}\\ 
        &\quad \text{(Directional derivative is less than norm of gradient)} \nonumber\\
        &=\int_\mathcal{C} \Vert \nabla_x f(\mathbf{r}) \Vert \sqrt{\mu^{x_i,x_j}(\mathbf{r})} \frac{1}{\sqrt{\mu^{x_i,x_j}(\mathbf{r})}}  d \mathbf{r} \\
        &\leq \sqrt{\int_\mathcal{C}  \Vert \nabla_x f(\mathbf{r}) \Vert ^2 \mu^{x_i,x_j}(\mathbf{r}) d \mathbf{r} } \sqrt{\int_\mathcal{C}{ \frac{1}{\mu^{x_i,x_j}(\mathbf{r})} d \mathbf{r}} } \label{inequality-2} \\
        &\quad \text{(Cauchy–Schwarz inequality)} \nonumber \\
        &\leq \sqrt{\int_\mathcal{C} \frac{1}{\mu^{x_i,x_j}(\mathbf{r})}  d \mathbf{r}} \label{inequality-3} \\
        &\quad \text{($f \in \mathcal{F}_S (\mathcal{X},\mu^{x_i,x_j})$, see Eq.~(14) in paper)} \nonumber \\
        &= \sqrt{\int_0^1 \frac{1}{\mu^{x_i,x_j}(\mathbf{r(t)})} |\mathbf{r}'(t)| d t} \nonumber\\
        &=\sqrt{\int_0^1 \Vert x_i - x_j \Vert |\mathbf{r}'(t)| d t} \nonumber\\
        & \quad \text{(see Eq.~\eqref{app:eq:sample-mu} )} \nonumber \\
        & = \Vert x_i-x_j \Vert \nonumber \\
        & \quad \text{($\mathbf{r}(t)=tx_i+(1-t)x_j$)} \nonumber \\
    \end{align}
    
    Let $F_{S}=\{f| \, f(x_i) - f(x_j) \leq \Vert x_i-x_j\Vert, \,\, \forall x_i \sim P_r, \forall x_j \sim P_g \}$.
    
    Let $A=\{(x_i,x_j) \,|\, x_i \in \mathcal{S}_r, x_j \in \mathcal{S}_g \}$ and $I_A=\begin{cases}
    1, \quad (x_i,x_j) \in A; \\
    0, \quad otherwise
    \end{cases}$.
    
    Let $A^c$ denote the complementary set of $A$ and define $I_{A^c}$ accordingly.
        
    For $\forall \pi \in \Pi(P_r, P_g)$, we have the following:
    \begin{align}
    W_{S}(P_r,P_g) &\leq {\sup}_{f\in F_{S}} \,\, \mathbb{E}_{x \sim P_r} \, [f(x)] - \mathbb{E}_{x \sim P_g} \, [f(x)]  \\ \nonumber
    & = {\sup}_{f\in F_{S}} \,\, \mathbb{E}_{(x_i,x_j) \sim \pi} [f(x_i)-f(x_j)]  \\ \nonumber
    & = {\sup}_{f\in F_{S}} \,\, \mathbb{E}_{(x_i,x_j) \sim \pi} [(f(x_i)-f(x_j)) I_A] + \mathbb{E}_{(x_i,x_j) \sim \pi} [(f(x_i)-f(x_j)) I_{A^c}] \\\nonumber
    & = {\sup}_{f\in F_{S}} \,\, \mathbb{E}_{(x_i,x_j) \sim \pi} [(f(x_i)-f(x_j)) I_A] \\\nonumber
    & \leq \mathbb{E}_{(x_i,x_j) \sim \pi} [\Vert x_j - x_i \Vert I_A]  \\\nonumber
    & \leq \mathbb{E}_{(x_i,x_j) \sim \pi} [\Vert x_i - x_j \Vert].
    \end{align}
    
    $W_{S}(P_r,P_g)\leq \mathbb{E}_{(x,y) \sim \pi} [\Vert x_i - x_j \Vert], \forall \pi \in \Pi(P_r, P_g)$ 
    
    $\Rightarrow W_{S}(P_r,P_g)\leq \inf_{\pi \in \Pi(P_r,P_g)} \, \mathbb{E}_{(x,y) \sim \pi} \, [\Vert x_i -x_j \Vert] = W_1(P_r,P_g)$. 
    
    (iii) Combining (i) and (ii), we have $W_{KR}(P_r,P_g)\leq W_{S}(P_r,P_g)\leq W_1(P_r,P_g)$. 
    
    Given $W_{KR}(P_r,P_g)=W_1(P_r,P_g)$, we have $W_{KR}(P_r,P_g)=W_{S}(P_r,P_g)=W_1(P_r,P_g)$.
\end{proof}

\subsection{Proof of Proposition 1}
\label{app:proof:prop1}

\begin{proposition}
	\label{app:prop:grad-sdual}
    Let $\pi^*$ be the optimal coupling in Eq.~\eqref{eq_w_primal_app}, then the optimal function $f^*$ in Sobolev duality Eq.~\eqref{eq_s_dual_form_app2} satisfies that: let $x_t = tx_i + (1-t)x_j$ with $0 \leq t \leq 1$, if $f^*$ is differentiable and $\pi^*(x, x)=0$ for all $x$, then it holds that $P_{(x_i, x_j)\sim \pi^*}[\nabla f^* (x_t)=\frac{x_i-x_j}{\Vert x_i-x_j \Vert}]=1$. 
\end{proposition}

\begin{lemma} \label{lemma2}
Let $\pi^*$ be the optimal transport plan in Eq.~(\ref{eq_w_primal_app}) and $f^*$ be the optimal discriminative function in Eq.~(\ref{eq_s_dual_form_app2}). It holds that 
\begin{equation} \label{w_dual_p}
{P_{(x_i, x_j)\sim \pi^*}} \Big[f^*(x_i)-f^*(x_j) = \Vert x_i - x_j \Vert \Big] = 1. 
\end{equation}
\end{lemma} 

\begin{proof} $ $\newline 

    
    (i) According to \cite{krdual}, the optimal $f$ in KR duality Eq.~\eqref{eq_w_dual_form_app1} satisfies the property in Eq.~\eqref{w_dual_p}. 
    
    In addition, as illustrated in Section 3.2.1, 1-Lipschitz condition in Eq.~\eqref{eq_w_dual_form_app1} is the sufficient condition of Sobolev duality Eq.~\eqref{eq_s_dual_form_app2}. Therefore, the function space of $f$ in KR duality Eq.~\eqref{eq_w_dual_form_app1} is a subspace of the function space defined in Sobolev duality Eq.~\eqref{eq_s_dual_form_app2}. 

    (ii) Recently, \cite{LGAN} proposed another relaxed dual form of Wasserstein distance. It states that:
    \begin{equation}
    \begin{aligned}
        W_{L}&(P_r,P_g) = {\sup}_{f} \,\, \Big\{ \mathbb{E}_{x_i \sim P_r} \, [f(x_i)] - \mathbb{E}_{x_j \sim P_g} \, [f(x_j)] \Big\},  \, \\
        &\emph{s.t.} \, f(x_i) - f(x_j) \leq \Vert x_i - x_j \Vert, \,\, \forall x_i \sim \mathcal{S}_r, \forall x_j \sim \mathcal{S}_g.
    \end{aligned}
    \label{new_constraint}
    \end{equation}
    
    It is worthy noticing that any $f^*$ under this constraint corresponds to one $f^*$ in Eq.~\eqref{eq_w_dual_form_app1} with the value of $f^*$ on $\mathcal{S}_r$ and $\mathcal{S}_g$ unchanged. Thus, any $f^*$ under the constraint in Eq.~\eqref{new_constraint} also holds the property in Eq.~\eqref{w_dual_p}. 
    

    According to Part (ii) of our proof in Section~\ref{app:proof:section-sobolev-dual}, given the constraint $f \in \mathcal{F}_S (\mathcal{X}, \mu ^{x_i,x_j})$ in Sobolev duality Eq.~\eqref{eq_s_dual_form_app2}, we can get the constraint in Eq.~\eqref{new_constraint} that $f(x_i) - f(x_j) \leq \Vert x_i - x_j \Vert$. Therefore, the restriction in Sobolev duality is the sufficient condition of Eq.~\eqref{new_constraint}, \textit{i.e.}, the function space of $f$ in Sobolev duality Eq.~\eqref{eq_s_dual_form_app2} is a subspace of the function space defined in Eq.~\eqref{new_constraint}. 
    
    (iii) Combining (i) and (ii), we have that 
    \begin{itemize}
    	\item The function space of $f$ in KR duality Eq.~\eqref{eq_w_dual_form_app1} is a subspace of the function space defined in Sobolev duality Eq.~\eqref{eq_s_dual_form_app2}, and the function space of $f$ in Sobolev duality Eq.~\eqref{eq_s_dual_form_app2} is a subspace of the function space defined in Eq.~\eqref{new_constraint}.  
    	\item $f^*$ in both Eq.~\eqref{eq_w_dual_form_app1} and Eq.~\eqref{new_constraint} satisfies the property in Eq.~\eqref{w_dual_p}.
    \end{itemize} 
    Therefore, we conclude that the optimal $f^*$ in Sobolev duality Eq.~\eqref{eq_s_dual_form_app2} also holds this property. 
\end{proof}

\begin{theorem}
    If $f \in \mathcal{F}_S (\mathcal{X}, \mu ^{x_i,x_j})$, given $f^*(x_i)-f^*(x_j) = \Vert x_i - x_j \Vert $, then for $x_t=tx_i+(1-t)x_j$ with $0 \leq t \leq 1$, we have $P_{(x_i,x_j)\sim \pi^*}[\nabla f^* (x_t)  = \frac{x_i-x_j}{\Vert x_i-x_j \Vert}]=1$.
\end{theorem}
\begin{proof} $ $\newline

    Given $f^*(x_i)-f^*(x_j) =  \Vert x_i - x_j \Vert $, then the inequality in Eq.~\eqref{inequality-1}, Eq.~\eqref{inequality-2} and Eq.~\eqref{inequality-3} are turned to equations:
    
    (i) The first equation corresponding to Eq.~\eqref{inequality-1} is that $\int_\mathcal{C} \mathbf{v} \cdot \nabla_x f(\mathbf{r}) d \mathbf{r} = \int_\mathcal{C} \Vert \nabla_x f(\mathbf{r}) \Vert d \mathbf{r} $. It indicates that
    \begin{equation}
    \label{sobolev-direction}
        \nabla_x f(\mathbf{r}) = k\mathbf{v} = k\frac{x_i-x_j}{\Vert x_i-x_j \Vert},
    \end{equation}
    where $k$ is a scalar, \textit{i.e.}, we have that
    \begin{equation}
    	\nabla_x f(\mathbf{r}) \parallel \frac{x_i-x_j}{\Vert x_i-x_j \Vert}
    \end{equation}
    
    (ii) The second equation corresponding to Eq.~\eqref{inequality-2} is that $\int_\mathcal{C} \Vert \nabla_x f(\mathbf{r}) \Vert \sqrt{\mu^{x_i,x_j}(\mathbf{r})} \frac{1}{\sqrt{\mu^{x_i,x_j}(\mathbf{r})}}  d \mathbf{r} = \sqrt{\int_\mathcal{C}  \Vert \nabla_x f(\mathbf{r}) \Vert ^2 \mu^{x_i,x_j}(\mathbf{r}) d \mathbf{r} } \sqrt{\int_\mathcal{C}{ \frac{1}{\mu^{x_i,x_j}(\mathbf{r})} d \mathbf{r}} }$. According to Cauchy–Schwarz inequality, it indicates that $\Vert \nabla_x f(\mathbf{r}) \Vert \sqrt{\mu^{x_i,x_j}(\mathbf{r})}$ and $\frac{1}{\sqrt{\mu^{x_i,x_j}(\mathbf{r})}}$ are linearly dependent. More precisely, we have: 
    \begin{equation}
    \begin{aligned}
    \label{sobolev-scalar}
        &\Vert \nabla_x f(\mathbf{r}) \Vert \sqrt{\mu^{x_i,x_j}(\mathbf{r})} = k\frac{1}{\sqrt{\mu^{x_i,x_j}(\mathbf{r})}}, \\
        &\Vert \nabla_x f(\mathbf{r}) \Vert = k\frac{1}{\mu^{x_i,x_j}(\mathbf{r})} = k \Vert x_i - x_j \Vert,
    \end{aligned}
    \end{equation}
    where $k$ is a scalar.
    
    (iii) The third equation corresponding to Eq.~\eqref{inequality-3} is that 
    \begin{equation}
    \label{sobolev-sphere}
        \sqrt{ \int _\mathcal{X} \Vert \nabla _x f(\mathbf{r}) \Vert ^2 \mu ^{x_i,x_j} (\mathbf{r}) dt } = 1
    \end{equation}
    
    Combine Eq.~\eqref{sobolev-scalar} and Eq.~\eqref{sobolev-sphere}, we can get the solution of $k$ in Eq.~\eqref{sobolev-scalar}: $k= \frac{1}{\Vert x_i-x_j \Vert}$.
    
    Assign the value of $k$ into Eq.~\eqref{sobolev-scalar} we get the solution of the gradient norms:
    \begin{equation}
    \label{sobolev-norm}
        \Vert \nabla_x f(\mathbf{r}) \Vert = 1.
    \end{equation}
    
    Combining the gradient direction in Eq.~\eqref{sobolev-direction}  and norm in Eq.~\eqref{sobolev-norm}, we have that:
    \begin{equation}
        \nabla_x f(\mathbf{r}) = 
        \frac{x_i-x_j}{\Vert x_i-x_j \Vert}
    \end{equation}
    To conclude, we showed that if $(x_i,x_j)$ have the property that  $f^*(x_i)-f^*(x_j) = \Vert x_i - x_j \Vert $, then $\nabla_x f(x_t) = \frac{x_i-x_j}{\Vert x_i-x_j \Vert}$. Since this happens with probability 1 under $\pi$, we know that
    \begin{equation}
    P_{(x_i,x_j)\sim \pi^*}[\nabla f^* (x_t)  = \frac{x_i-x_j}{\Vert x_i-x_j \Vert}]=1
    \end{equation}
    and we finished the proof.
\end{proof}

\section{Example for the convergence of Sobolev IPM}
\label{app:sec:example}

\paragraph{Example 1.} Let $P_r$ and $P_g$ be two disjoint 2-dimensional uniform distributions. Consider the following case:
\begin{equation}
\begin{aligned}
P_r(x_1,x_2) &= 
\begin{cases}
1, &-1\leq x_1\leq 0\ and\ -1\leq x_2\leq 0,\cr
0, &otherwise.
\end{cases} \\
P_g(x_1,x_2) &= 
\begin{cases}
1, &0\leq x_1\leq 1\ and\ 0\leq x_2\leq 1,\cr
0, &otherwise.
\end{cases} 
\end{aligned}
\end{equation}
	Eq.~(11) in the paper states the property of the optimal critic $f^*$ in Sobolev IPM that
\begin{equation}
\label{app:sobolev-f}
\nabla_x f^*(x)= \frac{1}{n \mathcal{S}_{\mu}(P_r,P_g)} \frac{D^- F_{P_g}(x)- D^-F_{P_r}(x)}{\mu(x)}.
\end{equation}
According to this property, we have that
\begin{align}
\frac{\partial}{\partial x_1} f^*(a,b) 
=\ &k \Big( D^{-1}F_g(a,b)-D^{-1}F_r(a,b) \Big)\nonumber \\
=\ &k\frac{\partial}{\partial x_2}(F_g(a,b)-F_r(a,b)) \nonumber \\
=\ &k\int_{-\infty}^a \Big(P_g(x_1,b)-P_r(x_1,b)\Big) dx_1 \nonumber \\
=\ &k\int_{-\infty}^a P_g(x_1,b) dx_1 > 0
\end{align}
where $k$ is a scalar equal to $\frac{1}{n \mathcal{S}_{\mu}(P_r,P_g)\mu(a,b)}$. 
Similarly, we also have that $\frac{\partial}{\partial x_2}f^*(a,b)>0$. Thus, for any point $(a,b)$ from $P_g$, we have
\begin{align}
\frac{\partial}{\partial x_1} f^*(a,b) > 0,\ \  \frac{\partial}{\partial x_2}f^*(a,b)>0.
\end{align}
This implies in this case the gradients guide fake data distribution $P_g$ to move along the positive direction of the $x$ and $y$ axes, which is the opposite direction to converge to $P_r$. Thus we have shown that Sobolev IPM will fail to converge in some cases. By contrast, according to Proposition~\ref{app:prop:grad-sdual}, the gradients of the proposed Sobolev duality always follow the optimal transport plan which guarantees the convergence. 

%
%

\section{Proof for the local convergence of SWGAN-CP}
\label{app:sec:cp-proof}

In this section, we present the proof of the local convergence property of proposed SWGAN-CP. This proof is build upon the results from the theory of discrete dynamical systems on previous works \cite{nagarajan2017gradient,mescheder2017numerics,mescheder2018training}. The discrete version of a basic convergence theorem for continuous dynamical systems from \cite{nagarajan2017gradient} can be found in Appendix A.1 in \cite{mescheder2017numerics}. This theorem allows us to make statements about training algorithms for GANs for finite learning rates.
Besides, \cite{mescheder2017numerics} analysed the convergence properties
of simultaneous and alternating gradient descent, and \cite{mescheder2018training} states some eigenvalue bounds that were derived.

Following the notation of \cite{nagarajan2017gradient}, the training objective for the two players in GANs can be described by an objective function of the form
\begin{equation}\label{eq:gan-loss}
\mathcal{L}(\theta, w)= \mathbb{E}_{p(z)}\left[f(D_w(G_\theta(z)))\right] \\
+ \mathbb{E}_{p_{data}(x)}\left[f(-D_w(x))\right]
\end{equation}
for some real-valued function $f$. Then in the toy case defined in Section \ref{sec:suboptimal}, this is simplified to:
\begin{equation}\label{eq:gan-loss-simple}
\mathcal{L}(\theta, w)= f(w\theta) + f(0).
\end{equation}
Now we consider to add the GAN loss with the gradient penalty regularization term used in SWGAN-CP
\begin{equation}
    \mathcal{L}_{cp} = - \lambda (w^4 + w^2),
\end{equation}
then we have the following lemma:

\begin{lemma}
	\label{lemma:gradient-penalty}
	The eigenvalues of the Jacobian of the gradient vector field for the SWGAN-CP at the equilibrium point are given by
	\begin{equation}
	\gamma_{1/2}=-\lambda\pm\sqrt{\lambda^{2}-f^{\prime}(0)^{2}}.
	\end{equation}
	In particular,  for $\lambda>0$ all eigenvalues
	have negative real part.
	Hence, simultaneous and alternating gradient descent are both locally convergent 
	for small enough learning rates.
\end{lemma}

\begin{proof}
	The gradient vector field for SWGAN-CP becomes
	\begin{equation}
	\tilde{v}(\theta, w)=\left(\begin{array}{c}
	-f^{\prime}(\theta w) w\\
	f^{\prime}(\theta w)\theta-\lambda(4w^3+2w)
	\end{array}\right).
	\end{equation}
	The Jacobian $\tilde{v}^{\prime}(\theta, w)$ is therefore given by
	\begin{equation}
	\begin{pmatrix}
	-f^{\prime\prime}(\theta w) w^{2} & -f^{\prime}(\theta w)-f^{\prime\prime}(\theta w)\theta w\\
	f^{\prime}(\theta w)+f^{\prime\prime}(\theta w)\theta w & f^{\prime\prime}(\theta w)\theta^{2}-\lambda(12w^2+2)
	\end{pmatrix}.
	\end{equation}
	Evaluating it at $\theta= w=0$ yields
	\begin{equation}
	\tilde{v}^{\prime}(0,0)=\left(\begin{array}{cc}
	0 & -f^{\prime}(0)\\
	f^{\prime}(0) & -2\lambda
	\end{array}\right)
	\end{equation}
	whose eigenvalues are given by
	\begin{equation}
	\gamma_{1/2}=-\lambda \pm \sqrt{\lambda^{2}-f^{\prime}(0)^{2}}.
	\end{equation}
\end{proof}

\newpage

\section{Algorithm for SWGAN-GP and WGAN-AL}

The training procedures of SWGAN-GP and WGAN-AL that we used in Section 5.2 are formally presented in Algorithm \ref{swgan-gp} and Algorithm \ref{wgan-al} respectively. 

\begin{algorithm}[!bthp]
	\caption{Sobolev Wasserstein GAN with Gradient Penalty}
	\label{swgan-gp}
	\textbf{Input}: $\rho$ penalty weight, $\eta$ learning rate, $k$ number of critic iterations per generator iteration, $n$ batch size, $m$ number of points sampled between each pair of fake and real data
	
	\textbf{Initial}: critic parameters $w$, generator parameters $\theta$
	
	\begin{algorithmic}[1] 
		\WHILE{$\theta$ has not converged}
		\FOR{$k=1,...,k$}
		\STATE Sample a minibatch $x_i$, $i=1,...,n$, $x_i\sim P_r$
		\STATE Sample a minibatch $z_i$, $i=1,...,n$, $z_i\sim P_z$
		\STATE $\widetilde{x}_i \leftarrow G_{\theta}(z_i)$, $i=1,...,n$
		\FOR{$i=1,...,n$}
		\STATE Sample a minibatch $\hat{x}_{ij} = t_jx_i + (1-t_j) \widetilde{x}_i$, $j=1,...,m$, $t_j\sim \text{U}(0,1)$
		\STATE $\Omega^{(i)} \leftarrow 1- \frac{1}{m} \sum_{j=1}^{m} \Vert\nabla_xD(\hat{x}_{ij}) \Vert ^2$
		\STATE $\mathcal{L}^{(i)}_{gp} \leftarrow - \lambda\,\big( \max\{-\Omega^{(i)}(D_w,G_\theta), 0\}\big) ^2$
		\ENDFOR
		\STATE $\mathcal{L} \leftarrow \frac{1}{n} \sum_{i=1}^{n} \big\{ D_w(x_i) - D_w(\widetilde{x}_i) + \mathcal{L}_{gp}^{(i)} \big\}$
		\STATE $g_w \leftarrow \nabla_{w} \mathcal{L} (w,\theta,\alpha)$
		\STATE $w \leftarrow w + \eta \text{ADAM}(w,g_w)$
		\ENDFOR
		\STATE Sample a minibatch $z_i$, $i=1,...,n$, $z_i\sim P_z$
		\STATE $d_{\theta} \leftarrow -\nabla_{\theta} \frac{1}{n} \sum_{i=1}^{n} D_w(G_{\theta}(z_i))$
		\STATE $\theta \leftarrow \theta - \eta \text{ADAM} (\theta,d_{\theta})$
		\ENDWHILE
	\end{algorithmic}
\end{algorithm}

\begin{algorithm}[!tbhp]
	\caption{Wasserstein GAN with Augmented Lagrangian}
	\label{wgan-al}
	\textbf{Input}: $\rho$ penalty weight, $\eta$ learning rate, $k$ number of critic iterations per generator iteration, $n$ batch size, $m$ number of points sampled between each pair of fake and real data
	
	\textbf{Initial}: critic parameters $w$, generator parameters $\theta$, Lagrange multiplier $\alpha$
	
	\begin{algorithmic}[1] 
		\WHILE{$\theta$ has not converged}
		\FOR{$k=1,...,k$}
		\STATE Sample a minibatch $x_i$, $i=1,...,n$, $x_i\sim P_r$
		\STATE Sample a minibatch $z_i$, $i=1,...,n$, $z_i\sim P_z$
		\STATE $\widetilde{x}_i \leftarrow G_{\theta}(z_i)$, $i=1,...,n$
		\FOR{$i=1,...,n$}
		\STATE Sample a minibatch $\hat{x}_{ij} = t_jx_i + (1-t_j) \widetilde{x}_i$, $j=1,...,m$, $t_j\sim \text{U}(0,1)$
		\STATE $h^{(ij)} = 1 - \Vert \nabla _x D_w(\hat{x}_{ij})\Vert$
		\STATE $\mathcal{L}_{al}^{(j)} \leftarrow \sum_{j=0}^m\big\{ \alpha h^{(ij)} - \frac{\rho}{2} (h^{(ij)})^2 \big\}$
		\ENDFOR
		\STATE $\mathcal{L} \leftarrow \frac{1}{n} \sum_{i=1}^{n} \big\{ D_w(x_i) - D_w(\widetilde{x}_i) + \mathcal{L}_{al}^{(i)} \big\}$
		\STATE $(g_{\alpha},g_w) \leftarrow (\nabla_{\alpha} \mathcal{L},\nabla_{w} \mathcal{L}) (w,\theta,\alpha)$
		\STATE $\alpha \leftarrow \alpha - \rho g_{\alpha}\ $ (SGD rule on $\alpha$ with learning rate $\rho$)
		\STATE $w \leftarrow w + \eta \text{ADAM}(w,g_w)$
		\ENDFOR
		\STATE Sample a minibatch $z_i$, $i=1,...,n$, $z_i\sim P_z$
		\STATE $d_{\theta} \leftarrow -\nabla_{\theta} \frac{1}{n} \sum_{i=1}^{n} D_w(G_{\theta}(z_i))$
		\STATE $\theta \leftarrow \theta - \eta \text{ADAM} (\theta,d_{\theta})$
		\ENDWHILE
	\end{algorithmic}
\end{algorithm}

\section{ResNet architecture}
\label{app:sec:net-architecture}
We adopt the same ResNet network structures and hyperparameters as \cite{wgan-gp}. The generator and critic are residual networks. \cite{wgan-gp} use pre-activation residual blocks with two $3 \times 3$ convolutional layers each and ReLU nonlinearity. Batch normalization is used in the generator but not the critic. Some residual blocks perform downsampling (in the critic) using mean pooling after the second convolution, or nearest-neighbor upsampling (in the generator) before the second convolution. 

Formally, we present our ResNet architecture in Table~\ref{table:resnet}. Further architectural details can be found in our open-source model in supplementary material. 

\begin{table}[!tphb]
	\footnotesize
	\centering
	\caption{ResNet architecture.}
	\label{table:resnet}
	\begin{tabular}{c  c  c c}
		\toprule
		\multicolumn{4}{c}{Generator $G(z)$}  \\
		\midrule
		Operation &Kernel size& Resample& Output Dims\\
		\midrule
		Noise & N/A & N/A & $128$ \\
		Linear & N/A & N/A & $128\times4\times4$ \\
		Residual block & $[3\times3]\times2$ & Up & $128\times8\times8$ \\
		Residual block & $[3\times3]\times2$ & Up & $128\times16\times16$ \\
		Residual block & $[3\times3]\times2$ & Up & $128\times32\times32$ \\
		Conv, tanh & $3\times3$ & N/A & $3\times32\times32$ \\
		\bottomrule
		\toprule
		\multicolumn{4}{c}{Critic $D(x)$}  \\
		\midrule
		Operation &Kernel size& Resample& Output Dims\\
		\midrule
		Residual block & $[3\times3]\times2$ & Down & $128\times16\times16$ \\
		Residual block & $[3\times3]\times2$ & Down & $128\times8\times8$ \\
		Residual block & $[3\times3]\times2$ & N/A & $128\times8\times8$ \\
		Residual block & $[3\times3]\times2$ & N/A & $128\times8\times8$ \\
		ReLU, mean pool & N/A & N/A & $128$ \\
		Linear & N/A & N/A & $1$ \\
		\bottomrule
	\end{tabular}
\end{table}

\newpage
\section{More generated samples}
\label{app:sec:generated-samples}

\subsection{MNIST samples}

\begin{figure}[H] 
	\centering
	\begin{subfigure}{0.6\linewidth}
		\vspace{10pt}
		\centering
		\includegraphics[width=0.99\columnwidth]{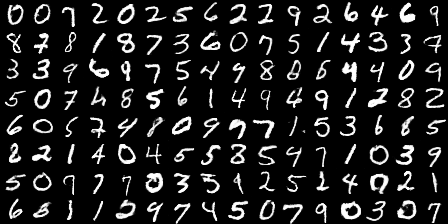}
	\end{subfigure}
	\begin{subfigure}{0.6\linewidth}
		\vspace{10pt}
		\centering
		\includegraphics[width=0.99\columnwidth]{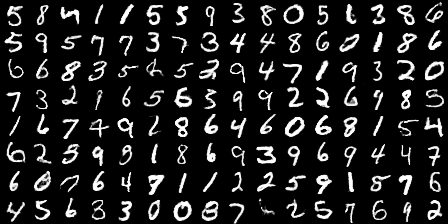}
	\end{subfigure}	
	\begin{subfigure}{0.6\linewidth}
		\vspace{10pt}
		\centering
		\includegraphics[width=0.99\columnwidth]{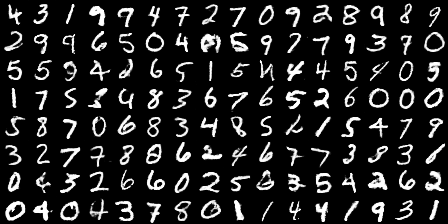}
	\end{subfigure}	
	\caption{Generated MNIST samples of SWGAN-AL.}
\end{figure}

\newpage

\subsection{CIFAR-10 samples}

\begin{figure}[H]
\vspace{10pt}
\centering
\begin{subfigure}{1\linewidth}
    \centering
    \includegraphics[width=0.4\linewidth]{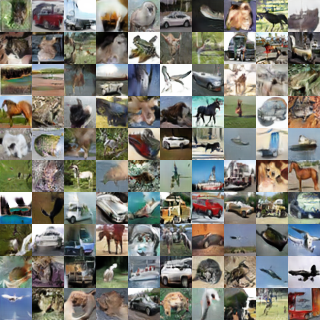}
    \hspace{+20pt}
    \centering
    \includegraphics[width=0.4\linewidth]{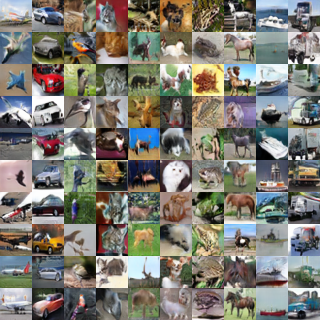}
\end{subfigure}	
\caption{Generated CIFAR-10 samples of SWGAN-AL. \textbf{Left}: Samples generated by our unsupervised model. \textbf{Right}: Samples generated by our conditional model. We add AC-GAN conditioning to our unconditional model. Samples from the same class are displayed in the same column.}
\hspace{+20pt}
\label{figure:cifar-al}
\end{figure} 

\begin{figure}[H]
\centering
\begin{subfigure}{1\linewidth}
    \centering
    \includegraphics[width=0.4\linewidth]{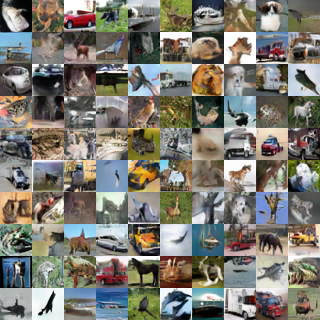}
    \hspace{+20pt}
    \centering
    \includegraphics[width=0.4\linewidth]{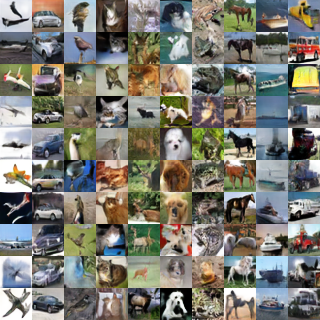}
\end{subfigure}	
    \caption{Generated CIFAR-10 samples of SWGAN-GP. \textbf{Left}: Samples generated by our unsupervised model. \textbf{Right}: Samples generated by our conditional model. We add AC-GAN conditioning to our unconditional model. Samples from the same class are displayed in the same column.}
\label{figure:cifar-gp}
\end{figure} 

\newpage

\subsection{Tiny-ImageNet samples}

\begin{figure}[H]
\centering
    \centering
    \includegraphics[width=0.55\linewidth]{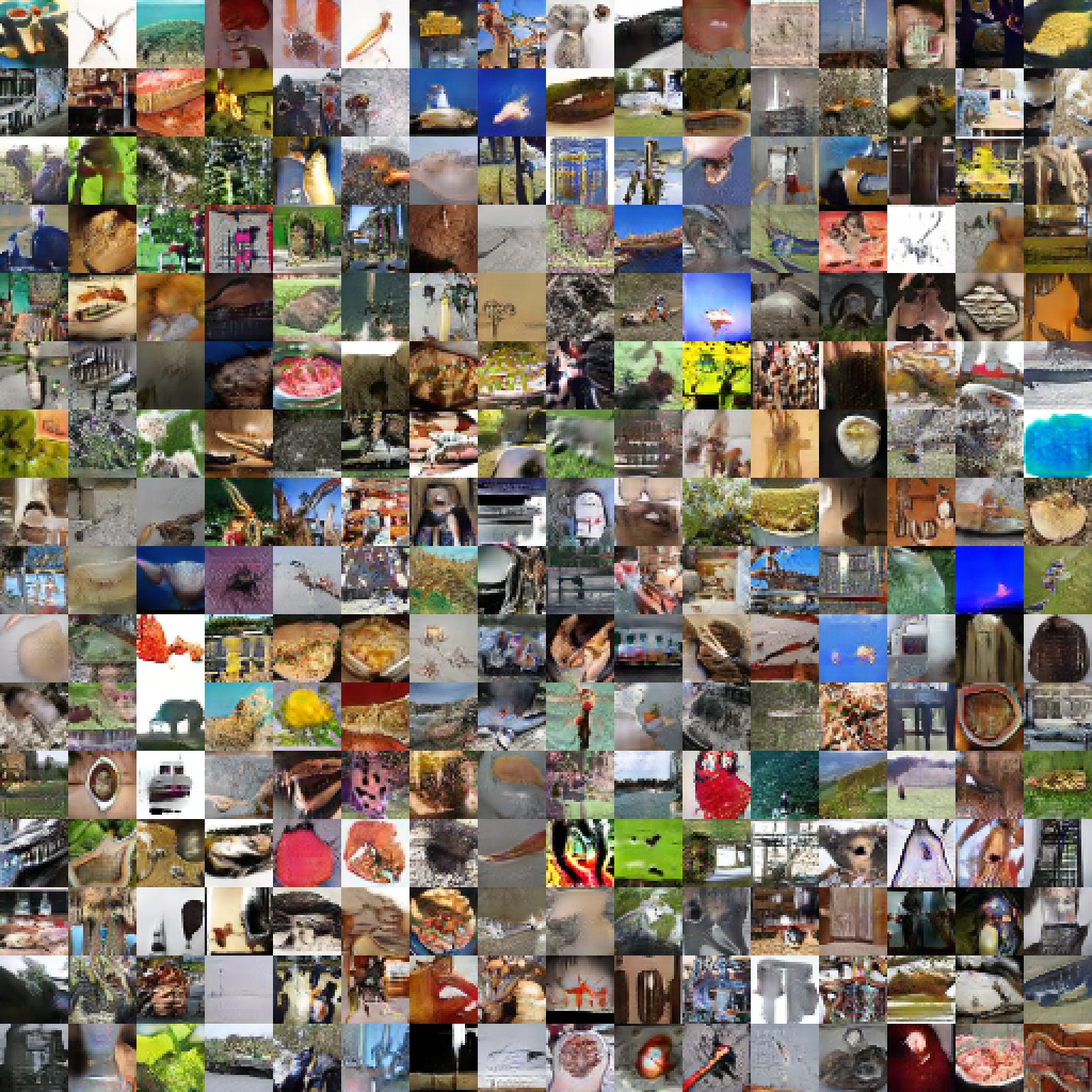}
    \caption{Generated Tiny-ImageNet samples of SWGAN-AL.}
\label{figure:tiny-al}
\end{figure} 

\vspace{-10pt}

\begin{figure}[H]
\centering
    \centering
    \includegraphics[width=0.55\linewidth]{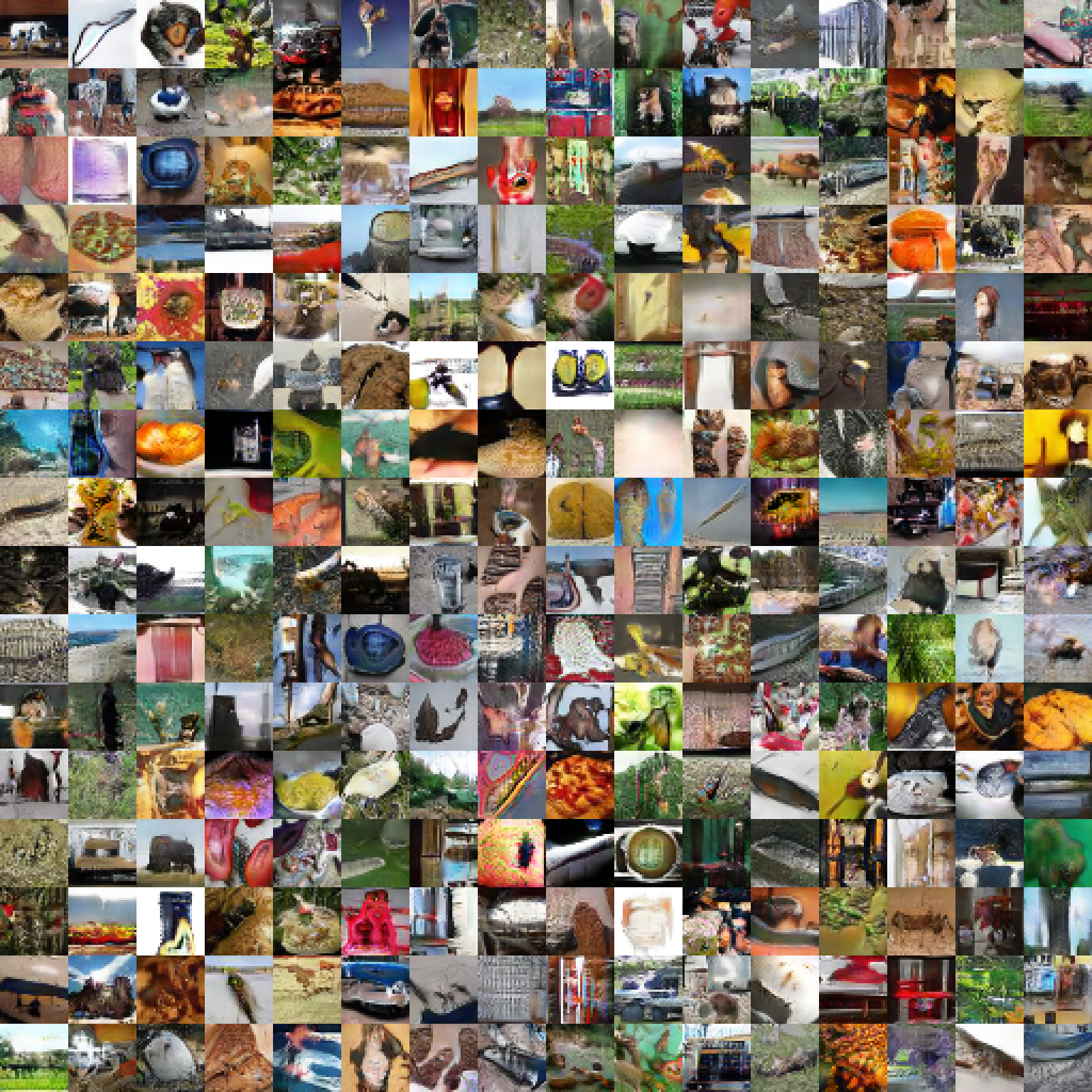}
    \caption{Generated Tiny-ImageNet samples of SWGAN-GP.}
\label{figure:tiny-gp}
\end{figure}



\end{document}